\documentclass{article}

\PassOptionsToPackage{numbers, compress}{natbib}
\usepackage[preprint]{neurips_2026}

\usepackage[utf8]{inputenc} 
\usepackage[T1]{fontenc}    
\usepackage{hyperref}       
\usepackage{url}            
\usepackage{booktabs}       
\usepackage{amsfonts}       
\usepackage{nicefrac}       
\usepackage{microtype}      
\usepackage[table]{xcolor}  
\usepackage{multirow}
\usepackage{subfig}
\usepackage{float}
\usepackage{overpic}
\usepackage{lineno}
\usepackage{amsmath}
\usepackage{bm}
\usepackage{algorithm2e}
\usepackage[most]{tcolorbox}
\usepackage{listings}
\usepackage{courier}
\usepackage[toc,page,title]{appendix}
\usepackage{natbib}
\usepackage{pifont}

\definecolor{red}{rgb}{1, 0, 0}
\definecolor{darkgreen}{rgb}{0, 0.5, 0}
\definecolor{highlightgreen}{rgb}{0.3, 0.6, 0}
\hypersetup{colorlinks=true, citecolor=darkgreen, linkcolor=red, urlcolor=red}
\definecolor{colthetapop}{HTML}{017100} 
\SetKwComment{Comment}{/* }{ */}
\RestyleAlgo{ruled}

\newtheorem{theorem}{Theorem}
\newtheorem{assumption}{Assumption}

\title{Addressing Performance Saturation for LLM RL via Precise Entropy Curve Control}

\author{%
  Bolian Li, Yifan Wang, Yi Ding, Anamika Lochab, Ananth Grama, Ruqi Zhang \\
  Purdue University, West Lafayette, IN, USA \\
  Correspondence to: \texttt{li4468@purdue.edu}, \texttt{ruqiz@purdue.edu} \\
}

\begin{document}

\maketitle

\begin{abstract}
Reinforcement learning (RL) has enabled complex reasoning abilities in large language models (LLMs). However, most RL algorithms suffer from performance saturation, preventing continued gains as RL training scales. This problem can be characterized by the collapse of entropy, a key diagnostic for exploration in RL. Existing attempts focus on preventing entropy collapse through regularization or clipping. However, their resulting entropy curves often exhibit instability in the long term, which hinders performance gains. In this paper, we introduce Entrocraft, a simple rejection-sampling approach that realizes \emph{user-customized entropy schedule} by biasing the advantage distributions. Entrocraft requires no objective regularization and is advantage-estimator-agnostic. Theoretically, we relate per-step entropy change to the advantage distribution under minimal assumptions. This explains the behavior of existing RL and entropy-preserving methods. Entrocraft also enables a systematic study of entropy schedules, which reveals that linear annealing, which starts high and decays to a slightly lower target, performs best. Empirically, Entrocraft addresses performance saturation, significantly improving generalization, output diversity, and long-term training. It enables a 4B model to outperform an 8B baseline, sustains improvement for up to 4× longer before plateauing, and raises pass@K by 50\% over the baseline.\footnote{The code is available at \url{https://github.com/lblaoke/entrocraft}.}\footnote{We also provide an interactive demo for playing with entropy curve control at \url{https://lblaoke.github.io/demo/entrocraft}.}
\end{abstract}

\section{Introduction}
Reinforcement learning (RL) has become the dominant approach for aligning with human preference and realizing multi-step reasoning ability in large language models (LLMs)~\citep{schulman2017proximal, bai2022training, ouyang2022training}. Despite these successes, many RL algorithms still underperform anticipated performance limits: as training scales, performance saturates earlier than expected, leaving additional data and compute unable to translate into further improvements~\citep{hou2024does, parkhorizon, beukman2026preventing}. A core reason behind this saturation is the collapse of the exploration–exploitation balance, where the LLM over-commits to a narrow region of solutions and stops exploring alternative reasoning trajectories~\citep{cui2025entropy, yuedoes, li2026back}. Empirically, this phenomenon is well captured by entropy dynamics: the frequently observed entropy collapse corresponds to a shrinking exploration ability during RL.

Recent efforts have resulted in several entropy-preserving techniques to prevent entropy drop during RL. These techniques are based on loss regularization~\citep{williams1991function}, clipping~\citep{yu2025dapo, cui2025entropy, wang2026entropy}, or positive-negative decoupling~\citep{zhusurprising, yang2025entropic}. While effectively increasing entropy, the entropy curves during RL training are still \emph{coarsely} controlled. Entropy can drift too high after a few steps, which in turn makes RL unstable and thus hinders sustained performance gains. Besides, they typically control entropy indirectly through the loss or update rule, making it difficult to prescribe an explicit entropy schedule over long training horizons. These drawbacks is particularly severe in long-term RL training. 

To address this, we propose Entrocraft, a method for precise control over the entropy curve that allows entropy schedules to be user-customized. Fig.~\ref{fig:overview} summarizes our method, the entropy curve control, and empirical improvements. We begin with an LLM-oriented theoretical analysis of entropy change based on realistic policy assumptions. We highlight that entropy changes are negatively related to the advantage, and high model confidence amplifies such entropy changes. 

Based on the theoretical results, we design a simple rejection sampling to filter out positive/negative-advantage rollout samples when entropy is lower/higher than a threshold, biasing the advantage distribution towards the entropy-increasing/decreasing region. Since rejection sampling directly modifies the advantage distribution, it is able to move the entropy to target values within very few steps, enabling the accurate crafting of entropy curves. The method requires no entropy regularization and applies as a drop-in to existing RL algorithms. 

Precise control opens a question that the field has not yet been able to ask experimentally: \emph{what entropy schedule is the best?} Comparing across schedule families, we find that a simple linear annealing schedule performs best.

The main contributions of this paper can be summarized below:
\begin{itemize}
    \item We provide rigorous theoretical results on entropy changes grounded in realistic LLM-based policy assumptions. Entropy changes are negatively related to the advantages and high model confidence amplifies such changes.
    \item We introduce a lightweight controller based on rejection sampling for entropy schedules in LLM RL. Unlike entropy regularization, clipping, or decoupling methods, Entrocraft does not modify the RL objective, and is advantage-estimator-agnostic. Entrocraft can craft the entropy curve to be user-specified entropy schedules, which is the key to addressing performance saturation.
    
    \item Extensive experiments demonstrate the effectiveness of Entrocraft. It significantly improves generalization (a 4B model surpasses an 8B baseline), increases output diversity (AIME-25 pass@32 is 50\% higher than baseline), and extends the training horizon (sustaining improvement for up to 4× longer before plateauing as training scales).
\end{itemize}

\begin{figure}[t]\vspace{-40pt}
    \centering
    \includegraphics[width=\linewidth]{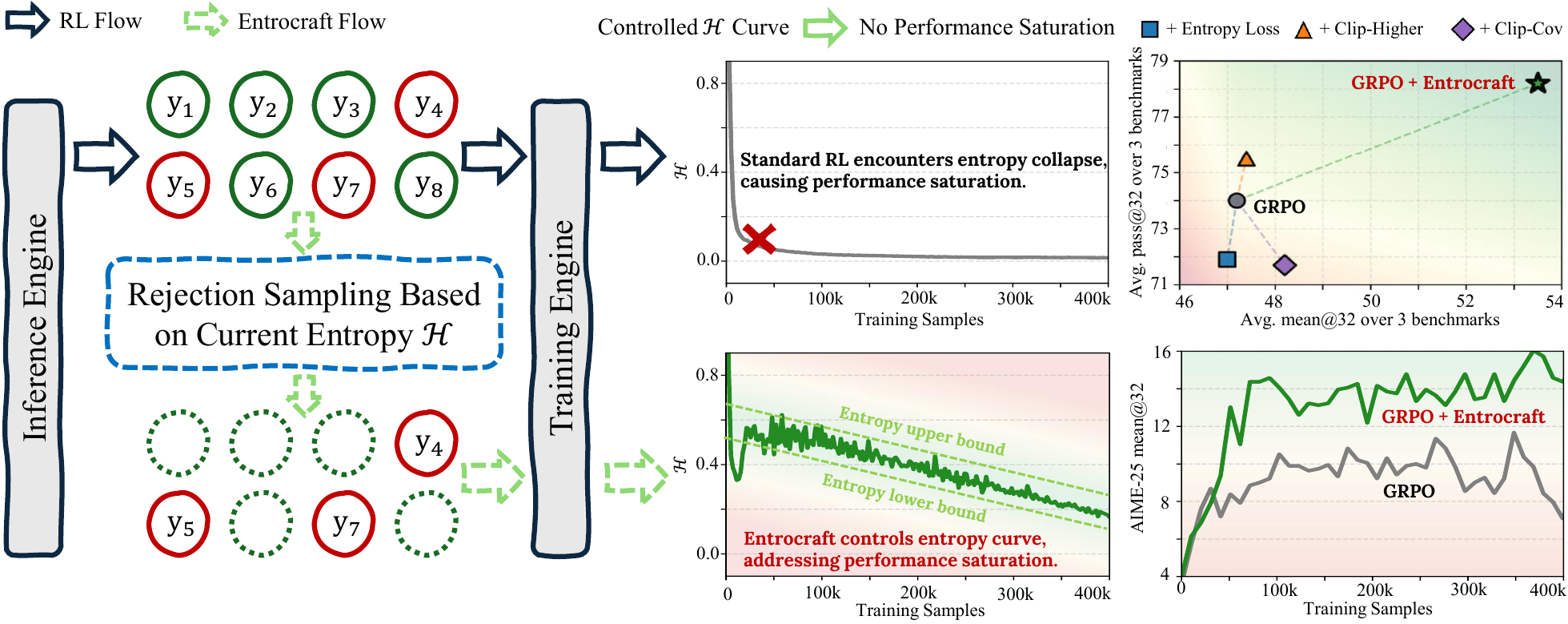}
    \caption{Overview of Entrocraft. It uses entropy-guided rejection sampling to filter rollouts against a target entropy, enabling precise control over the entropy curve throughout RL training. This control addresses performance saturation, a key obstacle to scaling RL. We find that a linear annealing schedule performs best, improving generalization, output diversity, and sustained training.}
    \label{fig:overview}
\end{figure}

\section{Preliminaries}
\subsection{Reinforcement Learning for LLMs}

In a standard policy-gradient RL framework like Group Relative Policy Optimization (GRPO)~\citep{shao2024deepseekmath} or Group Sequence Policy Optimization (GSPO)~\citep{zheng2025group}, the language model (or actor) we aim to train is denoted as a $\bm{\theta}$-parameterized distribution (or policy) $\pi_{\bm{\theta}}$. The direct output of language models is a softmax distribution over the entire vocabulary $\mathbb{V}$, interpreted as next-token probabilities: $\bm{p}_t = \pi_{\bm{\theta}}(\cdot|x,y_{<t})$. Each new token $y_t$ is drawn from $\bm{p}_t$.

Each RL step consists of rollout generation, advantage estimation, and policy update, allowing the model to explore different potential answers and learn from environment feedback. For a single question (or prompt) $x$, rollout generation samples a set of $G$ responses $\{y_i\}_{i=1}^G$ from an old checkpoint $\pi_{\bm{\theta}_{\text{sampler}}}(\cdot|x)$. The following PPO-style objective is used by many recent RL algorithms:
\begin{equation}\small
    \mathcal{J}(\bm{\theta}) = \frac{1}{G}\sum_{i=1}^G\sum_{t=1}^{|y_i|}\min\left[r_I(t)\cdot\hat{A}(x,y_i), \text{clip}\left(r_I(t),1-\epsilon_{\text{low}},1+\epsilon_{\text{high}}\right)\hat{A}(x,y_i)\right],
    \label{eq:grpo}
\end{equation}
where $r_I(t) = \frac{\pi_{\bm{\theta}}(y_{i,t}|x,y_{i,<t})}{\pi_{\bm{\theta}_{\text{sampler}}}(y_{i,t}|x,y_{i,<t})}$ is the importance sampling ratio, and $\hat{A}$ is the estimated advantage.

Our theoretical analysis is based on a simplified policy-gradient objective that does not consider clipping or importance sampling: $\mathcal{J}(\bm{\theta}) = \mathbb{E}_{y\sim\pi_{\bm{\theta}}(\cdot|x)}\hat{A}(x,y)$, and thus the per-step policy update is:
\begin{equation}
    \Delta\bm{\theta} = \eta \cdot \nabla_{\bm{\theta}}\mathcal{J}(\bm{\theta}) = \eta \cdot \mathbb{E}_{y\sim\pi_{\bm{\theta}}(\cdot|x)}[\hat{A}(x,y) \cdot \nabla_{\bm{\theta}}\log\pi_{\bm{\theta}}(y|x)],
    \label{eq:single_update}
\end{equation}
where $\eta$ is the learning rate.

\subsection{Entropy of LLMs}
The predictive entropy of LLMs provides a principled measurement of model uncertainty and serves as an indicator of response diversity and exploration capability. For a single question $x$ and answer $y$, the aggregated entropy is computed as: $\mathcal{H} = -\frac{1}{|y|}\sum_{t=1}^{|y|}\sum_{i=1}^{|\mathbb{V}|}p_{t,i}\log p_{t,i}$. The expected entropy, averaged over all prompts in a batch and their corresponding rollout samples, serves as an indicator of how LLMs' exploration capability evolves during RL. This evolution is known as entropy dynamics~\citep{renlearning, wang2026entropy}. In this paper, we primarily study entropy change during RL updates:
\begin{equation}
    \Delta\mathcal{H} = \mathcal{H}(\bm{p} + \delta\bm{p}) - \mathcal{H}(\bm{p}),
    \label{eq:entropy}
\end{equation}
to enable accurate and per-step entropy control.

\section{Theoretical Analysis: How Entropy Evolves during LLM RL}
This section presents theoretical results on entropy changes during RL training. We use these results to interpret the entropy dynamics of existing RL algorithms, particularly in long-running scenarios. Our analysis extends prior work~\citep{skydownacai2025rl, cui2025entropy, yang2025entropic, wang2026entropy, shen2025entropy} to a more realistic setting that does not require the actor to follow a tabular softmax policy\footnote{The tabular softmax policy assumes $\bm{\theta}=\bm{z}$, where logits are model parameters. However, in realistic LLM settings, the logits $\bm{z}$ are the functions of model parameters $\bm{\theta}$, and even a simple MLP module would make this assumption invalid.}. The resulting bounds are direct and easy to interpret, avoiding the complicated covariance and expectation terms that appear in prior analyses~\citep{skydownacai2025rl, wang2026entropy}.

\subsection{From Advantages to Entropy Changes}\label{sec:theory}
The analysis begins with two fundamental questions: (\romannumeral1) What is the sign of entropy change $\Delta\mathcal{H}$, and (\romannumeral2) what is its magnitude? These questions help us predict entropy change at each RL step, revealing how advantages affect entropy dynamics. To obtain exact analytical results, we make minimal assumptions about the actor policy and advantage distribution, only requiring that the learning rate is sufficiently low, as stated in Assumption~\ref{assumption}.

\begin{assumption}[Proximity of Policy Updates]
\label{assumption}
We assume the learning rate $\eta$ in Eq.~\eqref{eq:single_update} is small enough that the Taylor expansion approximation of policy probability updates holds (i.e., $\|\delta\bm{p}\|_2^2 \ll \|\delta\bm{p}\|_1$). This is a standard assumption in continuous optimization, and is satisfied in practice by modern adaptive optimizers like Adam~\citep{kingma2014adam} with typical learning rates (e.g., $10^{-6}\leq\eta\leq10^{-4}$).
\end{assumption}

\begin{theorem}[Token-Level Entropy Change in LLMs]
\label{theorem:token_entropy}
Consider a single policy-gradient update step of the form in Eq.~\eqref{eq:single_update}. Let $p_k$ be the probability that token $k$ is sampled during rollout generation. Then the sign entropy change $\Delta\mathcal{H}$ (Eq.~\eqref{eq:entropy}) triggered by token $k$ is opposite to that of its estimated advantage $\hat{A}_k$:
$$
\hat{A}_k \cdot \Delta\mathcal{H} \leq 0,~~\text{whenever}~~p_k > \prod_{i\in\mathbb{V}\backslash\{k\}} p_i^{-\frac{\delta p_i}{\delta p_k}},
$$
where $\delta p_i$ is the probability change at this RL step.
\end{theorem}

\begin{theorem}[Sequence-Level Entropy Change in LLMs]
\label{theorem:seq_entropy}
Consider a single policy-gradient update step of the form in Eq.~\eqref{eq:single_update}, and assume that all tokens share the same outcome reward. Let $p_{t,i} = \pi_{\bm{\theta}}(y_t=i|x,y_{<t})$ be the probability that $i$ is sampled as the $t$-th token in the sequence. The sign of the entropy change $\Delta\mathcal{H}$ (Eq.~\eqref{eq:entropy}) triggered by response $y$ is opposite to that of the estimated advantage $\hat{A}(x,y)$:
$$
\hat{A}(x,y) \cdot \Delta\mathcal{H} \leq 0,~~\text{whenever}~~\pi_{\bm{\theta}}(y|x) \geq \prod_{t=1}^{|y|}\prod_{i\in\mathbb{V}\backslash\{y_t\}} p_{t,i}^{-\frac{\delta p_{t,i}}{\delta p_{t,y_t}}},
$$
where $\delta p_{t,i}$ is the probability change at this RL step.
\end{theorem}

We provide theoretical guarantees in Theorem \ref{theorem:token_entropy} and \ref{theorem:seq_entropy} for token-level entropy and sequence-level entropy\footnote{NOTE: These theoretical results are based on the entropy computed from the learner policy $\pi_{\bm{\theta}}$.} respectively, and outline their proofs in Appendix~\ref{appendix:proof}. Intuitively, both theorems state that entropy changes are negatively related to the advantage, provided the probability of rollout samples is high enough to be above a baseline constant:
\begin{equation}\small
    \text{Entropy Change} \propto -~  \text{Advantage} \times (\text{Log Likelihood of Rollout Sample} - \text{Output Space Baseline}),
\end{equation}

where the \emph{output space baseline} is: $-\sum_{t=1}^{|y|}\sum_{i\in\mathbb{V}\backslash\{y_t\}}\frac{\delta \pi_{\bm{\theta}}(y_t=i|x,y_{<t})}{\delta \pi_{\bm{\theta}}(y_t|x,y_{<t})}\log \pi_{\bm{\theta}}(y_t=i|x,y_{<t})$.

Theorem~\ref{theorem:seq_entropy} suggests that positive-advantage rollout samples lead to an entropy drop if the model confidence $\pi_{\bm{\theta}}(y|x)$ is above the output space baseline. We further give empirical evidence to support this condition in Fig.~\ref{fig:output_space_baseline}, where we compare the log likelihoods (confidence) and output space baselines of training \texttt{Qwen3-4B-Base} under positive (RAFT++~\citep{xiong2025minimalist}), zero-mean (GRPO~\citep{shao2024deepseekmath}), and negative (NSR~\citep{zhusurprising}) advantage estimators. Fig.~\ref{fig:output_space_baseline} shows that the log likelihood is significantly higher than the output space baseline in all cases, verifying the condition for $\hat{A}(x,y) \cdot \Delta\mathcal{H} \leq 0$ to hold.

Entropy collapse/explosion in RL is a predictable consequence of advantage-weighted updates. Our results show that positive-advantage updates tend to reduce entropy, while negative-advantage updates tend to increase it. As a result, entropy collapse becomes the default, once training is dominated by positive advantages. This explanation also justifies the “accuracy-entropy tradeoff”~\citep{cui2025entropy} in standard RL algorithms, where accuracy increase leads to negative entropy changes. However, the theoretical results also suggest that entropy changes are not directly related to the model performance. It is possible to maintain entropy while still improving rewards if the algorithms selectively choose which advantage regions contribute to the policy gradients.

\begin{figure}[t]\vspace{-40pt}
    \centering
    \subfloat[All advantage estimators lead to sufficiently high confidence\label{fig:output_space_baseline}]{\includegraphics[width=.5\columnwidth,height=4cm]{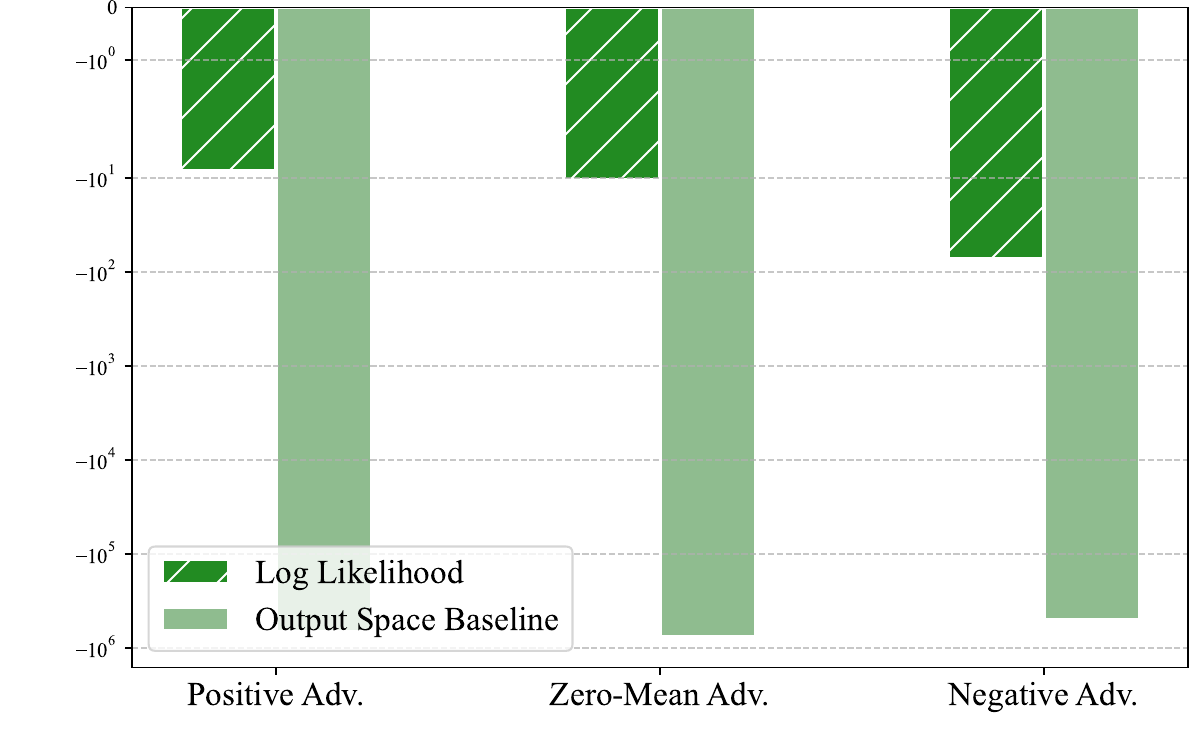}}
    \subfloat[Positive-sample confidence is consistently higher\label{fig:pos_ll_vs_neg_ll_gap}]{\includegraphics[width=.5\columnwidth,height=4cm]{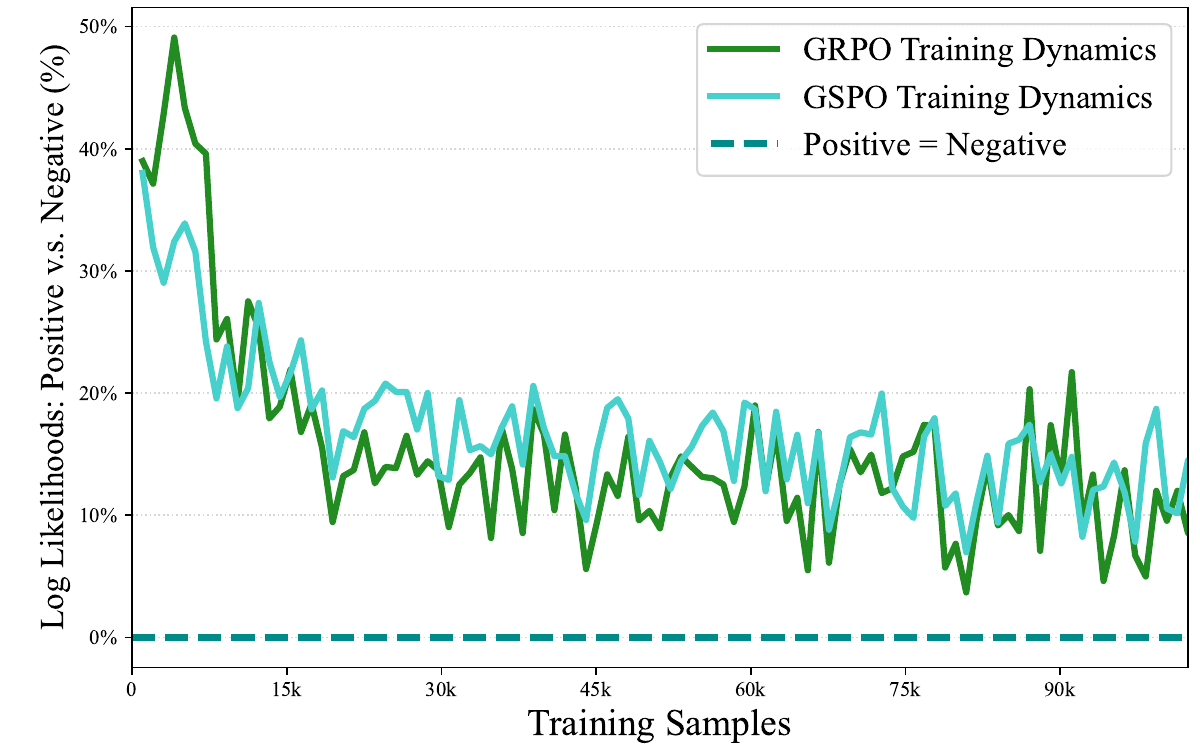}}
    \caption{Empirical justification for entropy analysis. (a) All advantage estimators lead to sufficiently large log likelihoods on average, so that the model confidence condition in Theorem~\ref{theorem:seq_entropy} always holds. (b) The log likelihoods of positive-advantage samples are always larger than those of negative samples, justifying why GRPO/GSPO encounters entropy collapse even with normalized advantages.}
\end{figure}

\begin{tcolorbox}[
  enhanced,
  colback=gray!20!white,
  colframe=black!20!darkgreen,
  boxrule=1pt,
  arc=1mm,
  left=2mm,
  right=2mm,
  top=1mm,
  bottom=1mm,
  fontupper=\normalsize,
  title=Takeaway for Entropy Analysis,
  colbacktitle=gray!20!white,
  coltitle=gray!10!white,
  fonttitle=\bfseries,
  boxed title style={
    colframe=black!20!darkgreen,
    colback=black!20!darkgreen,
    arc=1mm,
    boxrule=0pt,
  },
  attach boxed title to top left={
    xshift=3mm,
    yshift*=-\tcboxedtitleheight/2
  }
]
We prove the relationship between advantages and entropy changes for LLM-based policies. Entropy changes are negatively related to the advantage, and high model confidence will amplify such entropy changes.
\end{tcolorbox}

\subsection{Interpreting the Entropy Dynamics of Existing Methods}
The entropy dynamics of RL algorithms are important indicators of training stability and exploration-exploitation balance. Our theoretical results reveal a clear relationship between entropy change and advantage, explaining the entropy behavior of existing RL algorithms and their performance limitations. We discuss why existing RL algorithms exhibit specific entropy dynamics in the following discussion.

\paragraph{Standard RL Algorithms.}
Our theoretical results imply a categorization of existing RL algorithms that do not explicitly consider entropy. There are three types of algorithms based on their advantage statistics: (\romannumeral1) In positive-advantage RL like RAFT~\citep{dong2023raft} and RAFT++~\citep{xiong2025minimalist}, most RL steps lead to entropy drop; (\romannumeral2) in negative-advantage RL like NSR~\citep{zhusurprising}, most RL steps lead to entropy increase; (\romannumeral3) in zero-mean-advantage RL like GRPO~\citep{shao2024deepseekmath} and GSPO~\citep{zheng2025group}, empirical results show that the entropy still tends to decrease. We interpret this phenomenon by comparing the training dynamics of \texttt{Qwen3-4B-Base}, and find that this is due to the overconfidence of positive samples, as shown in Fig.~\ref{fig:pos_ll_vs_neg_ll_gap}. Models are consistently more confident in the positive samples, allowing the negative entropy changes to dominate the training dynamics.

\paragraph{Clipping.}
Many recent efforts leverage the clipping technique to address entropy collapse, including DAPO~\citep{yu2025dapo}, ADAPO~\citep{petrenkoentropy}, Clip$_{\mathcal{B}/\mathcal{V}}$~\citep{wang2026entropy}, and Clip-Cov~\citep{cui2025entropy}. The mechanism behind clipping is the removal of high-advantage and/or high-confidence tokens, which biases the advantage distribution toward 0-mean. Our theory explains why this works: it reduces expected $|\Delta\mathcal{H}|$ and thereby alleviates entropy drop.

\paragraph{Positive-Negative Decoupled RL.}
Recent studies also propose decoupled objectives for positive (correct) and negative (incorrect) rollout samples, respectively, inspired by the empirical finding that negative-only RL increases entropy~\citep{zhusurprising}. This approach is well explained by our theoretical framework, as it explicitly enforces the sign of advantages. For example, W-Reinforce~\citep{zhusurprising} modifies the coefficients of positive RL: $\mathcal{J}_{\text{W-Reinforce}} = \lambda\cdot\mathcal{J}_{\text{pos}} - \mathcal{J}_{\text{neg}}$, to weaken the entropy drop triggered by the positive objective; EntroPIC~\citep{yang2025entropic} further makes the coefficients adjustable: $\mathcal{J}_{\text{EntroPIC}} = (1+\alpha(\mathcal{H}))\cdot\mathcal{J}_{\text{pos}} - (1-\alpha(\mathcal{H}))\cdot\mathcal{J}_{\text{neg}}$, and eventually converges to a targeted entropy value. 

\section{Methodology}
In this section, we introduce our entropy-control framework (Entrocraft), which builds upon a simple rejection-sampling filter. We begin with rejection sampling in rollout generation (Section~\ref{sec:rejection_sample}), and then introduce the dynamic rejection sampling filter for entropy control (Section~\ref{sec:craft}). Finally, we discuss our insights on entropy curve annealing, highlighting that, for the first time, entropy in RL can be tuned just like learning-rate schedules (Section~\ref{sec:long_term}).

\subsection{Rejection Sampling as a Simple Entropy Controller}\label{sec:rejection_sample}
Our theoretical results suggest that entropy change is not directly tied to model performance. Entropy can remain stable or even increase while training accuracy improves, as long as the positive-advantage rollout samples are filtered out and no longer contribute to the policy gradients. This behavior can be realized by rejection sampling.

Our key observation is that entropy collapse/explosion is a consequence of uncontrolled gradient updates. From the theoretical results in Section~\ref{sec:theory}, the subset of rollouts contributing to the gradient determines whether an update is entropy-decreasing or entropy-increasing. The sign of $\Delta\mathcal{H}$ can be controlled by selecting which rollouts enter the policy gradient. Therefore, rather than developing new RL objectives or adding an auxiliary entropy loss, we find that a simple rejection-sampling filter at rollout generation suffices to precisely control entropy changes.

For example, to increase entropy, we apply rejection sampling to retain only the negative subset: ${\color{highlightgreen}\mathcal{S}_x} = \{y_i | \hat{A}(x,y_i) < 0\}$, and the RL training objective becomes:
\begin{equation}\small
    \mathcal{J}_{\text{rej}}(\bm{\theta}) = \frac{1}{|{\color{highlightgreen}\mathcal{S}_x}|}\sum_{y\in{\color{highlightgreen}\mathcal{S}_x}}\sum_{t=1}^{|y|}\min\left[r_I(t)\cdot\hat{A}(x,y), \text{clip}\left(r_I(t),1-\epsilon_{\text{low}},1+\epsilon_{\text{high}}\right)\hat{A}(x,y)\right],
\end{equation}
with the only difference from the standard RL objective (Eq.~\eqref{eq:grpo}) highlighted in {\color{highlightgreen}{color}}. Rejection sampling provides a simple, objective-agnostic entropy control knob, retaining the strengths of existing RL algorithms while eliminating the risk of entropy collapse or explosion. As it directly modifies the advantage distribution, the filter is responsive enough to move entropy to target values within a few steps, enabling the accurate crafting of entropy curves shown later in Section~\ref{sec:long_term}. The cost is comparable to or lower than standard RL, as only accepted samples contribute to the gradient computation. This can also be monitored by the effective rollout batch sizes as shown in Appendix~\ref{appednix:batch}.

\subsection{Stabilizing and Crafting Entropy Dynamics}\label{sec:craft}
Entropy dynamics have been used to monitor the training stability of RL~\citep{zheng2025stabilizing}. In a stable training run, the entropy curve should be within a reasonable range, neither low enough to trigger performance saturation~\citep{cui2025entropy, lu2024takes, kim2026training}, nor high enough to cause numerical overflow~\citep{zheng2025stabilizing}.

To realize stable entropy dynamics, we apply the rejection sampling filter to dynamically encourage or discourage the exploration of LLMs. The acceptance probability of rejection sampling depends on the current batch entropy $\overline{\mathcal{H}}$ against a target range $(h_{\text{low}},h_{\text{high}})$, in which we use an \emph{entropy out-of-range indicator}: $m = \mathbb{I}(\overline{\mathcal{H}}>h_{\text{high}}) - \mathbb{I}(\overline{\mathcal{H}}<h_{\text{low}})$ to measure the direction of entropy drift. When entropy is too low, the filter rejects most high-advantage rollouts while retaining lower- and negative-advantage ones. When entropy is too high, the filter retains positive-advantage rollouts and rejects most negative samples, steering RL updates toward entropy reduction. The full procedure is given in Algorithm~\ref{algo}.

Entrocraft provides a plug-and-play entropy control framework, applicable to all policy-gradient methods. It treats the entropy curve as a controllable training hyperparameter in the same spirit as a learning-rate schedule, making the training dynamics of RL stable and customizable.

\begin{algorithm}[t]
\caption{Entrocraft}
\label{algo}
\textbf{Inputs:} Question $x$, original rollout samples $\{y_i\}_{i=1}^G$, current policy $\pi_{\bm{\theta}}$, advantage estimator $\hat{A}$, and the entropy range $(h_{\text{low}},h_{\text{high}})$.

\textbf{Outputs:} The actual rollout samples used for the RL update ${\color{highlightgreen}\mathcal{S}_x}$.

${\color{highlightgreen}\mathcal{S}_x} \gets \emptyset$\;
Compute the current entropy $\overline{\mathcal{H}}$\;
$m \gets \mathbb{I}(\overline{\mathcal{H}}>h_{\text{high}}) - \mathbb{I}(\overline{\mathcal{H}}<h_{\text{low}})$ \Comment*[r]{entropy out-of-range indicator}
~\\
\For{$i = 1~..~G$} {
    \If{$m \cdot \hat{A}(x,y_i) \geq 0$} {
        ${\color{highlightgreen}\mathcal{S}_x} \gets \mathcal{S}_x \cup \{y_i\}$ \Comment*[r]{rejection sampling}
    }
}
Return ${\color{highlightgreen}\mathcal{S}_x}$.
\end{algorithm}

\subsection{Precise Entropy Curve Control and Annealing Schedules}\label{sec:long_term}
\paragraph{The Long-Term RL Challenge.} A growing body of work has shown that RL tends to sharpen the base policy around existing solutions rather than discover new ones~\citep{karan2025reasoning, zhangright, yuedoes, zhao2025echo}, a behavior consistent with the entropy collapse observed empirically. Once the policy becomes slightly more confident in a small subset of correct solutions, such solutions will be sampled more often, which further increases their likelihood. The problem is exacerbated in long-term RL. As training rewards rise, the advantage distribution becomes increasingly imbalanced and heavy-tailed, leaving fewer negative-advantage samples to counteract the drift. By Theorems~\ref{theorem:token_entropy} and \ref{theorem:seq_entropy}, these positive-advantage and high-likelihood solutions are mostly entropy-decreasing. The self-reinforcing feedback loop would lead to entropy collapse just within a few steps.

\paragraph{A Constant Entropy Target Is Not Enough.} This fragility motivated us to stress-test Entrocraft under long-term RL training. As demonstrated in Appendix~\ref{abl_study:failure_case}, Entrocraft with a slightly higher constant entropy target would become unstable and fluctuate a lot eventually. We attribute this instability to the imbalance of rollouts, which makes the negative samples so scarce in the long term that Entrocraft's entropy-increasing steps (rejecting all positive samples) rely on very few samples.

\paragraph{Curve Control with Annealing Schedules.} To address this, we propose to anneal the entropy curves as training proceeds. For example, we set the initial entropy target to be around 0.6, and gradually lower this target toward 0.2 during RL training. This can stabilize the training dynamics as it reduces the unstable entropy-increasing steps in the later phase of RL. We compare different annealing schemes in Section~\ref{sec:exp_anneal}, finding that the simple linear-decaying entropy curve achieves the best performance. This annealing design is uniquely enabled by Entrocraft. It converts entropy in RL from a passive training diagnostic into a controllable hyperparameter, extending the toolkit for tuning RL performance to any policy-gradient method.

\section{Experiments}
In this section, we present empirical results to demonstrate the effectiveness of the proposed Entrocraft algorithm. Specifically, we show a comprehensive benchmark comparison in Section~\ref{sec:exp_bench}, elaborate on the entropy curve annealing schemes in Section~\ref{sec:exp_anneal}, and discuss the case of long-term RL in Section~\ref{sec:exp_long_term}.

\subsection{Settings}\label{sec:setting}

\paragraph{Data and Models.}
The experiments described in this section focus on math reasoning tasks, using Numina-Math~\citep{numina_math_7b} (440K questions in total) as the training set. We hold out a 100K subset for general RL experiments, and the full-size dataset is used for long-term RL experiments. We primarily demonstrate the RL results using \texttt{Qwen3-4B-Base}~\citep{yang2025qwen3}, as well as the comparison with larger models (\texttt{Qwen3-8B-Base} and \texttt{Qwen3-14B-Base}) and models from different model families (\texttt{Llama-3.1-8B-Instruct})~\citep{grattafiori2024llama}.

\paragraph{RL Algorithms and Baselines.}
We primarily use the proposed Entrocraft algorithms to augment GRPO~\citep{shao2024deepseekmath} and GSPO~\citep{zheng2025group}. In comparison with Entrocraft, we also implement other entropy-preserving methods on top of these RL algorithms, including loss-regularization (entropy loss), clipping (Clip-Higher~\citep{yu2025dapo} and Clip-Cov~\citep{cui2025entropy}), and positive-negative decoupled RL (W-Reinforce~\citep{zhusurprising} and EntroPIC~\citep{yang2025entropic}).\footnote{The related clipping method ADAPO~\citep{petrenkoentropy} is not compared as it is primarily used for tool-use LLMs, incompatible with our experiments.} The implementation follows the standard verl framework~\citep{sheng2025hybridflow}. Other training details and hyper-parameters are summarized in Appendix~\ref{appendix:implement_detail}.

\paragraph{Evaluation.}
The evaluation scheme consists of AMC-23~\citep{knoveleng2025amc23}, and AIME-24/25/26~\citep{sun2025challenging}. Following previous works~\citep{xiong2025minimalist, zhusurprising, yang2025entropic}, we randomly sample 32 answers per question, with temperature set to 0.6. Due to space constraints, we show results from the full AIME experiments in Appendix~\ref{appendix:aime}.

\begin{table}[t]\small\setlength\tabcolsep{3pt}\renewcommand{\arraystretch}{0.8}\vspace{-40pt}
\caption{Overview of entropy-preserving baselines. We compare the methods along three properties: (\romannumeral1) Can the method reach a target entropy value? (\romannumeral2) Can it control entropy curves? (\romannumeral3) Does it apply to any policy-gradient method?}
\centering
\begin{tabular}{@{}l|ccc@{}}
\toprule
\multicolumn{1}{c|}{\textbf{Method}}       & \textbf{Reach Target Entropy} & \textbf{Control Entropy Curve} & \textbf{Algorithm-Agnostic} \\ \midrule
Entropy Loss                               &                             &                                & \checkmark                        \\
Clip-Higher~\citep{yu2025dapo}             &                             &                                & \checkmark                        \\
Clip-Cov~\citep{cui2025entropy}            &                             &                                & \checkmark                        \\
W-Reinforce~\citep{zhusurprising}          &                             &                                &                                   \\
EntroPIC~\citep{yang2025entropic}          & \checkmark                  &                                &                                   \\
\rowcolor{green!15!white}Entrocraft (ours) & \checkmark                  & \checkmark                     & \checkmark                        \\ \bottomrule
\end{tabular}
\end{table}

\subsection{Benchmark Evaluation}\label{sec:exp_bench}
We first conduct general RL experiments and evaluate the final checkpoints on math reasoning benchmarks, as shown in Table~\ref{tab:main} and Fig.~\ref{fig:scaling}. We demonstrate that the proposed Entrocraft outperforms all other baselines under both mean@32 and pass@32 settings. Fig.~\ref{fig:model_size_vs_mean32} highlights that Entrocaraft can enable \texttt{Qwen3-4B-Base} to outperform \texttt{Qwen3-8B-Base} trained using standard GRPO. The pass@K experiments in Fig.~\ref{fig:pass_at_k_curve_aime25}-\ref{fig:pass_at_k_curve_math500} also demonstrate that Entrocraft successfully prevents the actor from collapsing to just a few solutions, which benefits inference-time scaling.

\begin{table}[t]\small\setlength\tabcolsep{2pt}\renewcommand{\arraystretch}{0.85}\vspace{-50pt}
\caption{Evaluations on math reasoning benchmarks. The proposed Entrocraft consistently improves the final performance of RL algorithms, better than other entropy-preserving methods. The best scores are in \textbf{bold}, and the second-best scores are marked with \underline{underlines}.}
\label{tab:main}
\centering
\begin{tabular}{l|cc|cc|cc|cc}
\toprule
\multicolumn{1}{c|}{\multirow{2}{*}{\textbf{Method}}} & \multicolumn{2}{c|}{\textbf{MATH-500}} & \multicolumn{2}{c|}{\textbf{AMC-23}} & \multicolumn{2}{c|}{\textbf{AIME-25}} & \multicolumn{2}{c}{\textbf{Avg. Score}} \\
\multicolumn{1}{c|}{}                                 & mean@32            & pass@32           & mean@32           & pass@32          & mean@32           & pass@32           & mean@32            & pass@32            \\ \midrule
GRPO                                                  & 75.3               & 89.4              & 57.4              & 92.5             & 8.9               & 40.0              & 47.2               & 74.0               \\
~~~~ + Entropy Loss                                   & 76.0               & 89.8              & 55.0              & 92.5             & 9.9               & 33.3              & 47.0               & 71.9               \\
~~~~ + Clip-Higher                                    & 75.8               & 91.4              & 55.5              & \underline{95.0} & 10.8              & 40.0              & 47.4               & 75.5               \\
~~~~ + Clip-Cov                                       & 76.8               & 91.8              & 57.1              & 90.0             & 10.8              & 33.3              & 48.2               & 71.7               \\
\rowcolor{green!15!white}~~~~ + Entrocraft            & \textbf{79.0}      & \textbf{93.0}     & \textbf{65.0}     & \underline{95.0} & \textbf{15.1}     & \textbf{46.7}     & \textbf{53.5}      & \textbf{78.2}      \\ \midrule
GSPO                                                  & 75.6               & 90.2              & \underline{57.7}  & 92.5             & 8.9               & 33.3              & 47.4               & 72.0               \\
~~~~ + Entropy Loss                                   & 74.7               & 90.6              & 55.9              & \textbf{97.5}    & 11.9              & 33.3              & 47.5               & 73.8               \\
~~~~ + Clip-Higher                                    & 75.5               & 90.4              & 54.7              & \underline{95.0} & 10.0              & 40.0              & 46.7               & 75.1               \\
~~~~ + Clip-Cov                                       & 75.3               & 91.2              & 56.3              & 92.5             & 10.0              & 40.0              & 47.2               & 74.6               \\
\rowcolor{green!15!white}~~~~ + Entrocraft            & \underline{78.7}   & \underline{92.2}  & 56.1              & \underline{95.0} & \underline{14.7}  & 40.0              & \underline{49.8}   & \underline{75.7}   \\ \midrule
W-Reinforce                                           & 74.9               & 90.8              & 53.5              & 90.0             & 11.3              & 40.0              & 46.6               & 73.6               \\
EntroPIC                                              & 73.7               & 90.2              & 55.3              & 92.5             & 10.9              & \underline{43.3}  & 46.6               & 75.3               \\ \bottomrule
\end{tabular}
\end{table}

\begin{figure}[t]\vspace{-10pt}
    \centering
    \subfloat[Model Size v.s. AIME-25\label{fig:model_size_vs_mean32}]{\includegraphics[width=.33\columnwidth,height=3.2cm]{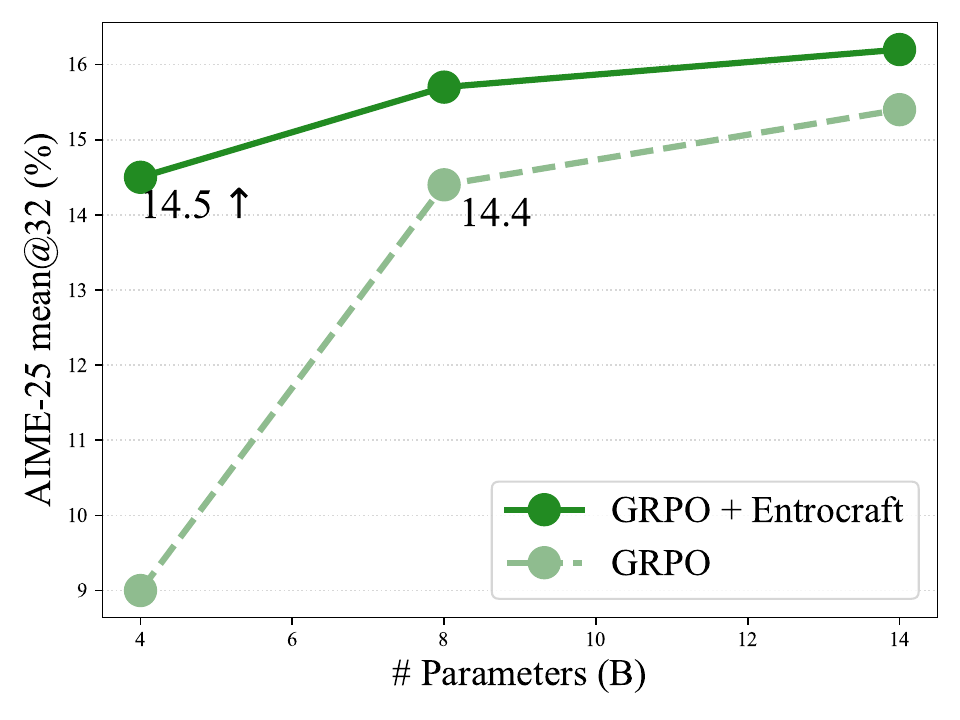}}
    \subfloat[pass@K curve on AIME-25\label{fig:pass_at_k_curve_aime25}]{\includegraphics[width=.33\columnwidth,height=3.2cm]{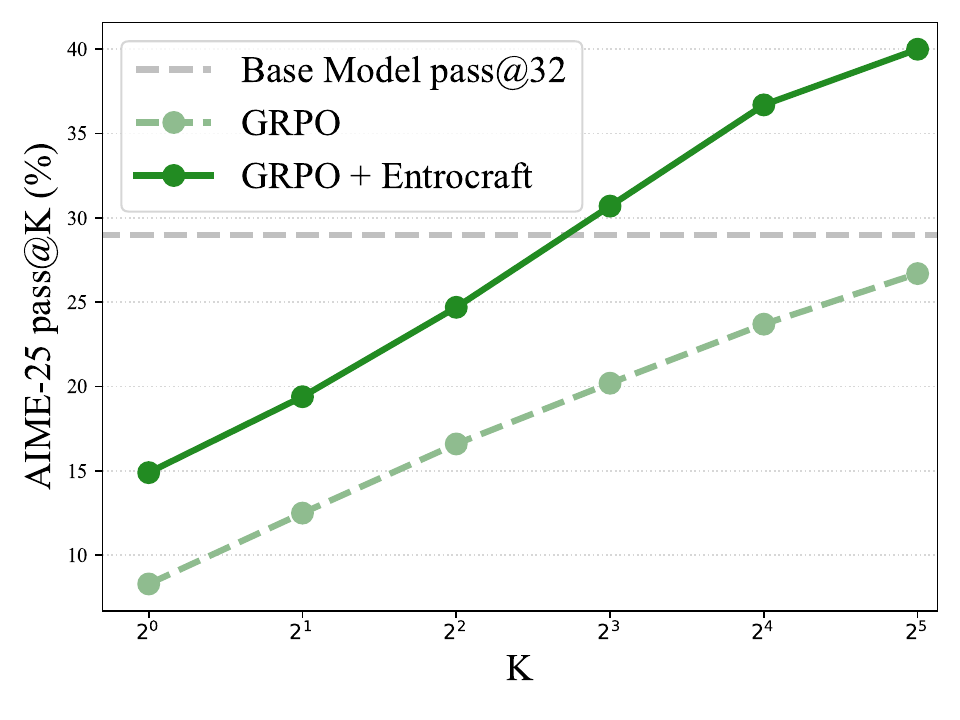}}
    \subfloat[pass@K curve on MATH-500\label{fig:pass_at_k_curve_math500}]{\includegraphics[width=.33\columnwidth,height=3.2cm]{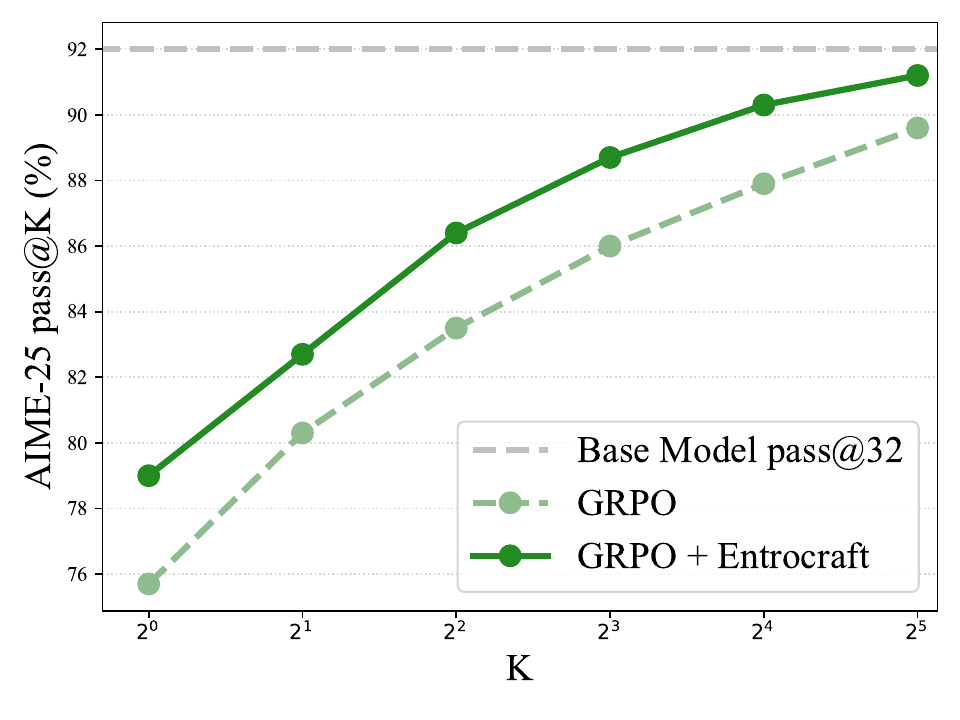}}
    \caption{Effectiveness of Entrocraft across model size and inference cost. (a) A 4B model can outperform an 8B model with Entrocraft; (b-c) Entrocraft improves inference-time scaling, with pass@K growing faster than under standard GRPO.}
    \label{fig:scaling}
\end{figure}

\subsection{Crafting Entropy Curves}\label{sec:exp_anneal}
We demonstrate the entropy control capability of Entrocraft. The entropy control powered by the proposed entropy-guided rejection sampling filter is powerful enough that entropy curves can be crafted directly, much like learning-rate schedules. As empirical evidence of entropy annealing introduced in Section~\ref{sec:long_term}, we comprehensively compare three entropy annealing schemes (fixed target, linear decay, and cosine decay) in Fig.~\ref{fig:decay_scheme}.\footnote{We set the fixed entropy target to be 0.5. Both linear and cosine decaying use annealing entropy range schedules, starting at (0.6, 0.7) and ending at (0.1, 0.2)} The fixed-target scheme, similar to previous entropy-control studies~\citep{yang2025entropic, petrenkoentropy}, becomes unstable after the first 200k training samples. Its entropy curve fluctuates sharply, leading to a drop in performance. In contrast, decaying schemes eliminate the instability and sustain improvement even after 400k training samples. 

\begin{figure}[t]\vspace{-50pt}
    \centering
    \subfloat[Training Reward]{\includegraphics[width=.33\columnwidth,height=3cm]{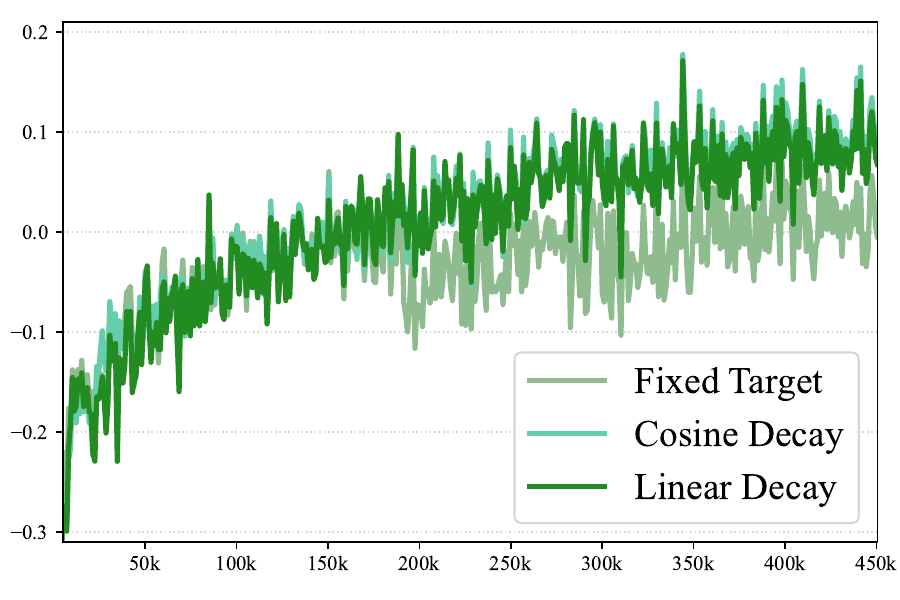}}
    \subfloat[Entropy]{\includegraphics[width=.33\columnwidth,height=3cm]{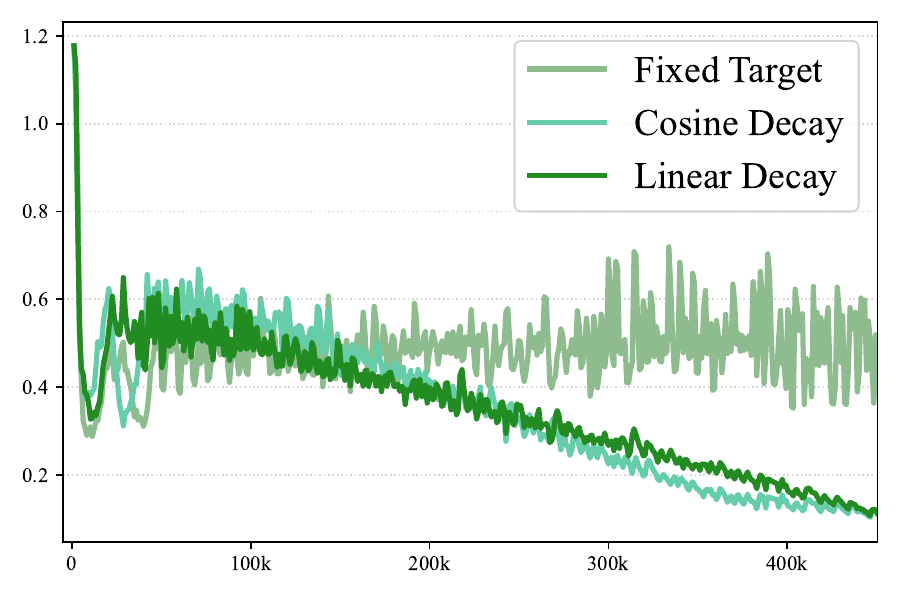}}
    \subfloat[KL Loss]{\includegraphics[width=.33\columnwidth,height=3cm]{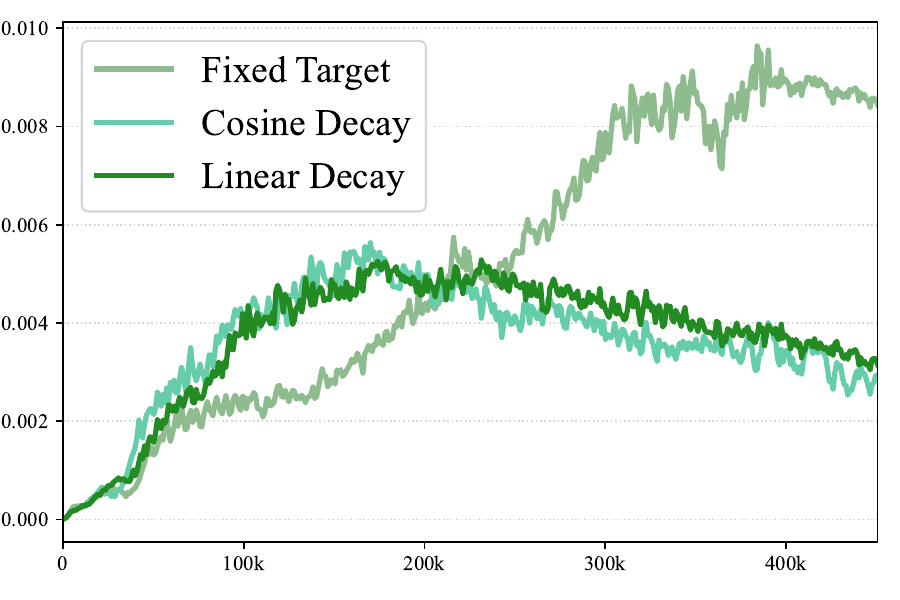}}
    \\
    \subfloat[MATH-500 mean@32]{\includegraphics[width=.33\columnwidth,height=3cm]{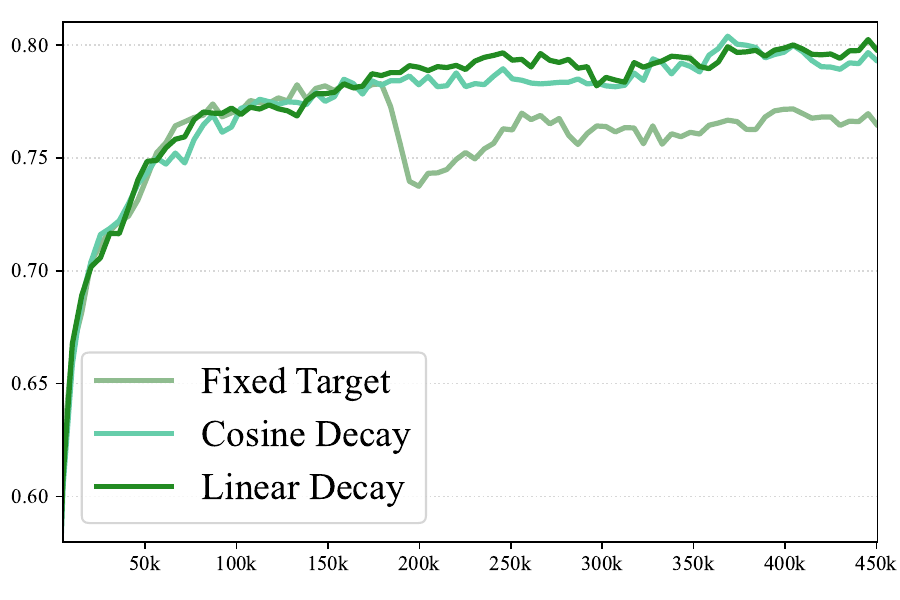}}
    \subfloat[AIME-25 mean@32]{\includegraphics[width=.33\columnwidth,height=3cm]{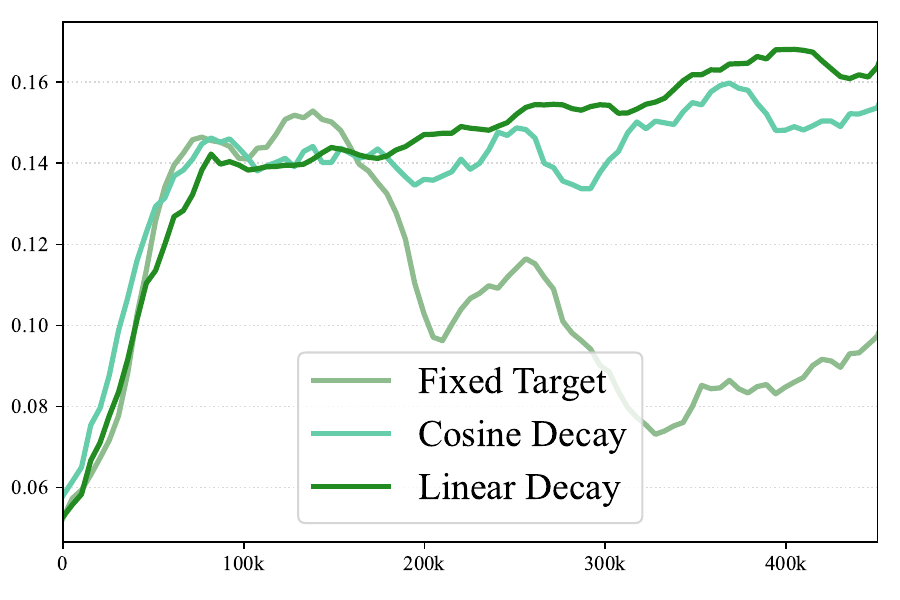}}
    \subfloat[AIME-26 mean@32]{\includegraphics[width=.33\columnwidth,height=3cm]{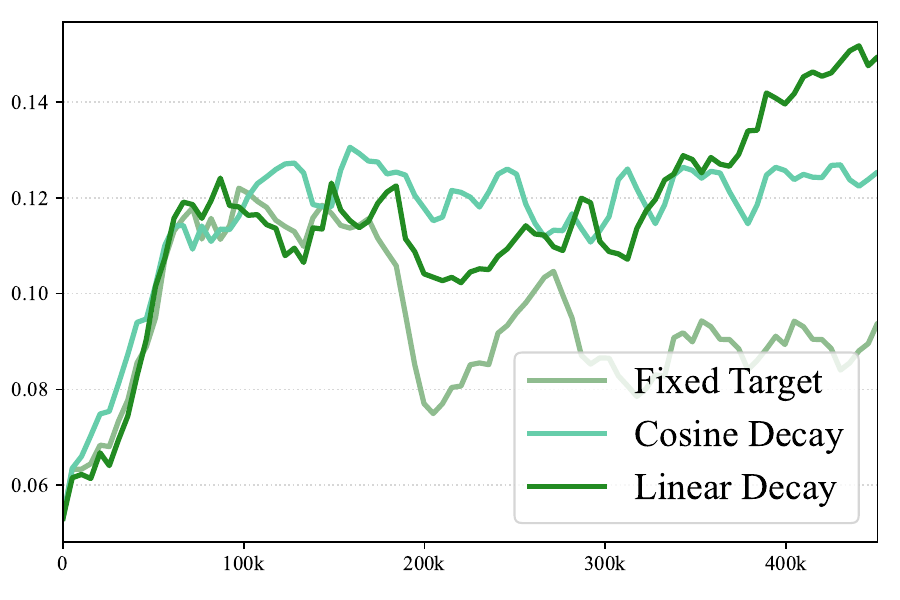}}
    \caption{Long-term training dynamics of three entropy annealing schemes implemented in Entrocraft. The fixed-target scheme becomes unstable due to rollout imbalance. Both linear and cosine decay schemes remain stable and sustain improvement, and linear decay is slightly better. x-axis: the number of samples used for training.}
    \label{fig:decay_scheme}
\end{figure}

\subsection{Long-Term RL}\label{sec:exp_long_term}
Finally, we demonstrate continual improvement in long-term RL powered by the proposed Entrocraft. As shown in Fig.~\ref{fig:long_term_curve} and Table~\ref{tab:long_term}, standard RL algorithms like GRPO~\citep{shao2024deepseekmath} improve smoothly through the first 100K training samples. However, we observe only minimal sustained gains thereafter, a phenomenon known as performance saturation~\citep{cui2025entropy, lu2024takes, kim2026training}. In contrast, entropy-preserving methods alleviate this saturation and achieve better final performance. Among all entropy-preserving methods, the proposed Entrocraft achieves the best long-term performance, surpassing all compared baselines at any of the 4 stages. We attribute this to its precise entropy control, which prevents entropy from drifting and avoids entropy instability common in uncontrolled entropy-preserving methods~\citep{yu2025dapo, cui2025entropy}. As a concrete illustration, Clip-Cov~\citep{cui2025entropy} suffers a performance drop after 300K training samples, caused by entropy explosion (Fig.~\ref{fig:entropy_4}) that destabilizes training in the later phase.

\begin{figure}[ht]\vspace{-15pt}
    \centering
    \subfloat[Entropy\label{fig:entropy_4}]{\includegraphics[width=.33\columnwidth,height=3cm]{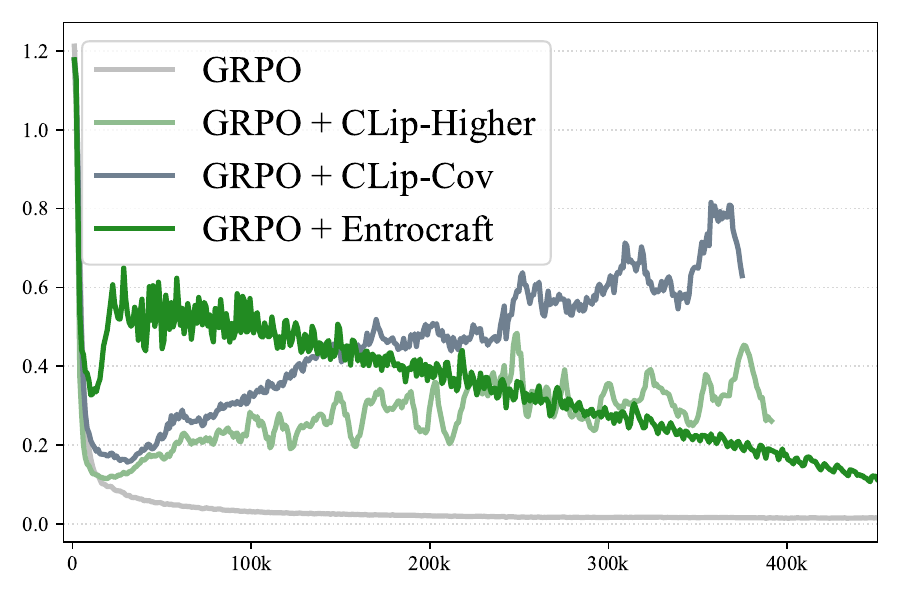}}
    \subfloat[MATH-500 mean@32]{\includegraphics[width=.33\columnwidth,height=3cm]{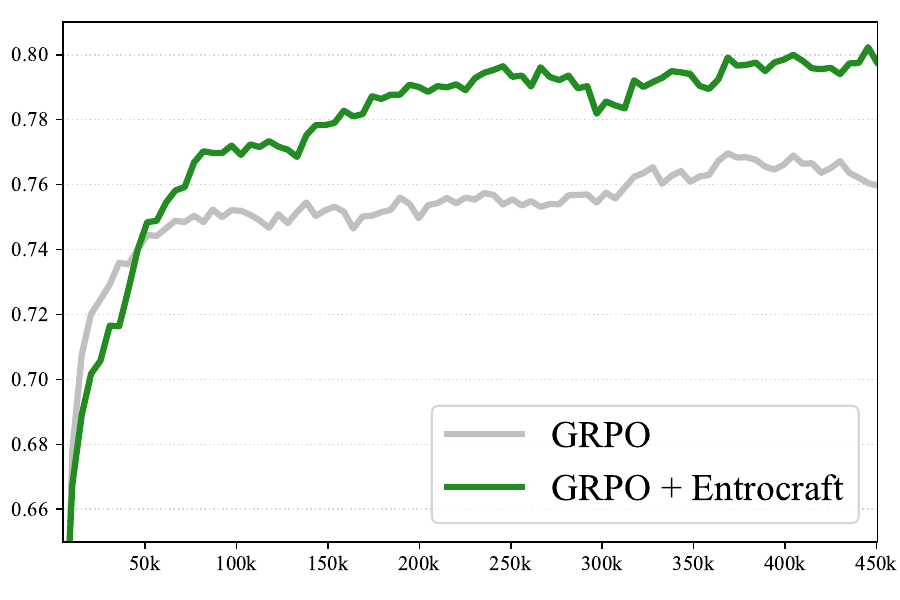}}
    \subfloat[AIME-25 mean@32]{\includegraphics[width=.33\columnwidth,height=3cm]{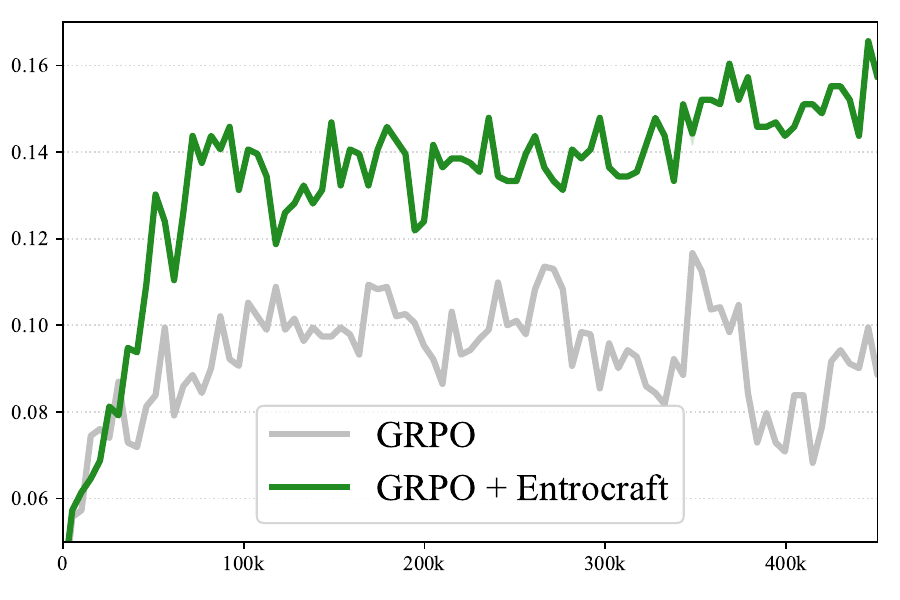}}
    \caption{Long-term performance comparison. Entrocraft accurately controls the entropy dynamics to be linearly decaying, and thus prevents the performance saturation and significantly improves over GRPO. x-axis: the number of samples used for training.}
    \label{fig:long_term_curve}
\end{figure}

\begin{table}[ht]\small\renewcommand{\arraystretch}{0.85}\vspace{-15pt}
\caption{Long-term RL comparison at the first 100-400K training samples. GRPO suffers from performance saturation after 100K samples due to entropy collapse, while entropy-preserving methods overcome this saturation. Among all, Entrocraft is the most stable and well-performing one.}
\label{tab:long_term}
\centering
\begin{tabular}{l|cccc|cccc}
\toprule
\multicolumn{1}{c|}{\multirow{2}{*}{\textbf{Method}}} & \multicolumn{4}{c|}{\textbf{MATH-500} mean@32}                & \multicolumn{4}{c}{\textbf{AIME-25} mean@32}                  \\
\multicolumn{1}{c|}{}                                 & 100K          & 200K          & 300K          & 400K          & 100K          & 200K          & 300K          & 400K          \\ \midrule
GRPO                                                  & 75.2          & 75.4          & 75.6          & 76.6          & 9.9           & 10.1          & 9.6           & 9.4           \\
~~~~ + Clip-Higher                                    & 76.0          & 77.6          & 78.4          & 78.8          & 9.6           & 12.3          & 12.8          & 13.6          \\
~~~~ + Clip-Cov                                       & 76.5          & 78.1          & \textbf{79.7} & 77.5          & 10.9          & 12.5          & 14.2          & 15.0          \\
\rowcolor{green!15!white}~~~~ + Entrocraft            & \textbf{76.9} & \textbf{78.9} & 79.3          & \textbf{79.9} & \textbf{13.1} & \textbf{14.2} & \textbf{14.8} & \textbf{16.2} \\ \bottomrule
\end{tabular}
\end{table}

\section{Conclusion}
This paper introduces Entrocraft, a simple and precise entropy-control method that addresses entropy collapse and the consequent performance saturation in RL. We first provide an LLM-oriented theoretical analysis of what drives entropy change, then design a rejection-sampling-based method that accurately controls entropy by biasing the advantage distribution. Experiments demonstrate that Entrocraft enables highly customizable entropy curves and consistently outperforms existing entropy-preserving methods across benchmarks. Entrocraft integrates as a drop-in to any policy-gradient method, enabling continual improvement on more data without saturation.

\section*{Ethics and Broader Impact Statement}\label{sec:ethics}
The paper does not involve human-subject data collection, personally identifiable information, or deployment in safety-critical settings. Potential risks include enabling stronger reasoning models that may also be misused in broader downstream applications. However, the paper primarily contributes a training-stability technique rather than a new capability domain. We document datasets, implementation details, hyperparameters, and compute requirements, supporting external scrutiny while discouraging inappropriate use.


\small{
\bibliography{neurips_2026}
\bibliographystyle{plainnat}
}


\newpage
\appendix
\small

\section{Discussion}
\subsection{Related Works}
\paragraph{Reinforcement Learning with Verifiable Rewards}
Reinforcement Learning~\citep{kaelbling1996reinforcement} has become the dominant approach in the post-training of LLMs, in which LLMs generate a set of rollout trajectories (exploration) and then enforce the rewarded trajectories while punishing the less-rewarded ones (exploitation). Defining the reward scores for LLMs’ trajectories is a critical challenge. Reinforcement Learning from human feedback (RLHF)~\citep{ouyang2022training} separately trains a reward model (RM) as the proxy of human preference. However, RM-based reward has been shown to suffer from significant reward hacking problems~\citep{weng2024rewardhack}, which makes the LLMs poorly generalized to new questions. Luckily, post-training of LLMs has increasingly focused on tasks that exhibit long, complex trajectories with simple verification, including math reasoning~\citep{cobbe2021training, grotschla2025benchmarking} and code generation~\citep{chen2021evaluating, jainlivecodebench}. For example, the famous o1~\citep{jaech2024openai} and R1~\citep{guo2025deepseek} use RL with verifiable rewards to elicit the complex reasoning capability of LLMs, and the GRPO~\citep{shao2024deepseekmath} has become the default RL algorithm for training many reasoning models. Recent works propose many variants of GRPO to mitigate its drawbacks, including the length inflation (Dr. GRPO~\citep{liuunderstanding} and GFPO~\citep{shrivastavasample}) and training instability (GSPO~\citep{zheng2025group} and \citet{zheng2025stabilizing}). The key challenge among these drawbacks is the scalability of RL algorithms: with more data and compute, will RL algorithms continually improve the model performance?

\paragraph{Entropy Dynamics of Reinforcement Learning}
The entropy has been used as a key indicator for models' exploration during RL~\citep{williams1991function, mnih2016asynchronous, haarnoja2017reinforcement}, which is essential for achieving the upper limit of model performance~\citep{sutton1998reinforcement}. The entropy collapse phenomenon has long been regarded as the default behavior of RL algorithms, trading exploration potential for higher rewards~\citep{cui2025entropy, karan2025reasoning, zhangright}. However, recent synthetic theoretical analysis reveals that the entropy collapse is related to the bias of the advantage distribution, and it is possible to improve performance without an entropy drop~\citep{skydownacai2025rl, cui2025entropy, yang2025entropic, wang2026entropy, shen2025entropy}. Further, this paper (Entrocraft) and a concurrent work~\citep{petrenkoentropy} prove the relationship between entropy changes and advantages under realistic LLM settings. Empirically, the techniques used to avoid entropy collapse include entropy loss\footnote{\citet{shen2025entropy} indicates that entropy loss only works well under traditional RL tasks where the discrete action space is small, and this effect is marginal for LLM RL.}, clipping (Clip-Higher~\citep{yu2025dapo}, Clip-Cov~\citep{cui2025entropy}, and Clip$_{\mathcal{B}/\mathcal{V}}$~\citep{wang2026entropy}), and positive-negative decoupled RL (W-Reinforce~\citep{zhusurprising}, EntroPIC~\citep{yang2025entropic},  and ADAPO~\citep{petrenkoentropy}). However, the entropy control implemented using the above methods is still not responsive enough, which makes the entropy curves still an observation-based metric, not customizable. In contrast, the proposed Entrocraft makes the entropy control accurate enough for entropy curve crafting, turning entropy dynamics from an observation-only metric to a hyperparameter like learning-rate schedules.

\subsection{Limitation and Future Directions}\label{appendix:limit_future}
This paper focuses on the stability and continual improvement of RL algorithms, and has validated the method's effectiveness on standard math reasoning tasks with dense models. However, in more challenging settings like multi-turn RL and mixture-of-expert (MoE) models, the entropy instability is more catastrophic~\citep{xu2025epo, zheng2025stabilizing} due to the more sparse solution space and the variable policies. The current method design is optimized for single-turn math reasoning task. We plan to extend the entropy analysis and entropy-control algorithm to such settings.

\section{Proofs}\label{appendix:proof}
\subsection{Proof of Theorem~\ref{theorem:token_entropy}}
\paragraph{Token-Level Entropy Change of LLMs.}
Consider a single RL update step. Let $p_k$ be the probability that token $k$ is sampled during rollout generation. The sign of the entropy change triggered by token $k$ is bounded by the sign of the advantage score:

$$
\hat{A}_k \cdot \Delta\mathcal{H} \leq 0,
$$

with probability $1 - \prod_{i\in\mathbb{V}\backslash\{k\}} p_i^{-\frac{\delta p_i}{\delta p_k}}$, where $\delta p_i$ is the probability change at this RL step.

\textbf{\emph{Proof:}}

Let $\bm{p}$ be the probability vector before the RL update and $\delta\bm{p}$ be the update. The entropy change can be approximated by Taylor expansion~\citep{abramowitz1948handbook}:

\begin{equation}
\begin{aligned}
    \Delta\mathcal{H} &= \mathcal{H}(\bm{p}+\delta\bm{p}) - \mathcal{H}(\bm{p}) \\ 
    &= \sum_{i\in\mathbb{V}}\frac{\partial\mathcal{H}}{\partial p_i} \delta p_i + O(\|\delta\bm{p}\|_2^2) \\ 
    &= -\sum_{i\in\mathbb{V}}(1+\log p_i)\cdot\delta p_i + O(\|\delta\bm{p}\|_2^2) \\ 
    &\overset{\text{(a)}}{=} -\sum_{i\in\mathbb{V}}\delta p_i \log p_i + O(\|\delta\bm{p}\|_2^2).
\end{aligned}
\end{equation}

Here, (a) is due to the constraint on output probabilities: $\sum_{i\in\mathbb{V}}p_i \equiv 1$. Then, we identify the condition for entropy to decrease (i.e., $\Delta\mathcal{H} < 0$). This occurs when:

\begin{equation}
    \sum_{i\in\mathbb{V}}\delta p_i \log p_i > 0.
\end{equation}

For token $k$ with positive advantage $\hat{A}_k > 0$, the RL update increases its probability: $\delta p_k > 0$. By the probability conservation constraint $\sum_{i\in\mathbb{V}}\delta p_i \equiv 0$, we have $\sum_{i\in\mathbb{V}\backslash\{k\}}\delta p_i = -\delta p_k < 0$. Then, the entropy change can be rewritten as:

\begin{equation}
    \Delta\mathcal{H} = -\delta p_k \log p_k - \sum_{i\in\mathbb{V}\backslash\{k\}}\delta p_i \log p_i + O(\|\delta\bm{p}\|_2^2).
\end{equation}

Since $\delta p_i = \delta p_k \frac{\delta p_i}{\delta p_k}$ for $i\neq k$, we obtain:

\begin{equation}
\begin{aligned}
    \Delta\mathcal{H} &= -\delta p_k \left(\log p_k + \sum_{i\in\mathbb{V}\backslash\{k\}}\frac{\delta p_i}{\delta p_k}\log p_i\right) + O(\|\delta\bm{p}\|_2^2) \\ 
    &= -\delta p_k \left(\log p_k - \log\prod_{i\in\mathbb{V}\backslash\{k\}} p_i^{-\frac{\delta p_i}{\delta p_k}}\right) + O(\|\delta\bm{p}\|_2^2).
\end{aligned}
\end{equation}

Entropy decreases ($\Delta\mathcal{H} < 0$) when the term in parentheses is positive: $\log p_k > \log\prod_{i\in\mathbb{V}\backslash\{k\}} p_i^{-\frac{\delta p_i}{\delta p_k}}$. This condition holds with probability $1 - \prod_{i\in\mathbb{V}\backslash\{k\}} p_i^{-\frac{\delta p_i}{\delta p_k}}$, which approaches 1 as probability mass concentrates on token $k$. Symmetrically, for tokens with negative advantage, entropy increases. Therefore, $\hat{A}_k \cdot \Delta\mathcal{H} \leq 0$ holds with high probability.

\hfill $\square$

\subsection{Proof of Theorem~\ref{theorem:seq_entropy}}
\paragraph{Sequence-Level Entropy Change of LLMs.}
Consider a single RL update step and assume all tokens share the same outcome reward. Let $p_{t,i} = \pi_{\bm{\theta}}(y_t=i|x,y_{<t})$ be the probability that $i$ is sampled as the $t$-th token in the sequence. The sign of the entropy change triggered by response $y$ is bounded by the sign of the advantage score:

$$
\hat{A}(x,y) \cdot \Delta\mathcal{H} \leq 0,~~~~\text{if}~~\pi_{\bm{\theta}}(y|x) \geq \prod_{t=1}^{|y|}\prod_{i\in\mathbb{V}\backslash\{y_t\}} p_{t,i}^{-\frac{\delta p_{t,i}}{\delta p_{t,y_t}}},
$$

where $\delta p_{t,i}$ is the probability change at this RL step.

\textbf{\emph{Proof:}}

Similar to the token-level entropy, by Taylor expansion~\citep{abramowitz1948handbook}, the sequence-level entropy change can be expressed as:

\begin{equation}
\begin{aligned}
    \Delta\mathcal{H} &= \sum_{t=1}^{|y|}\sum_{i\in\mathbb{V}}\frac{\partial\mathcal{H}}{\partial p_{t,i}}\cdot\delta p_{t,i} + O\left(\sum_{t=1}^{|y|}\|\delta\bm{p}_t\|_2^2\right) \\ 
    &= -\frac{1}{|y|}\sum_{t=1}^{|y|}\sum_{i\in\mathbb{V}}(1+\log p_{t,i})\cdot\delta p_{t,i} + O\left(\sum_{t=1}^{|y|}\|\delta\bm{p}_t\|_2^2\right) \\ 
    &\overset{\text{(b)}}{=} -\frac{1}{|y|}\sum_{t=1}^{|y|}\sum_{i\in\mathbb{V}}\delta p_{t,i}\cdot\log p_{t,i} + O\left(\sum_{t=1}^{|y|}\|\delta\bm{p}_t\|_2^2\right) \\ 
    &= -\frac{1}{|y|}\sum_{t=1}^{|y|}\delta p_{t,y_t} \left(\log p_{t,y_t} + \sum_{i\in\mathbb{V}\backslash\{y_t\}}\frac{\delta p_{t,i}}{\delta p_{t,y_t}}\cdot\log p_{t,i}\right) + O\left(\sum_{t=1}^{|y|}\|\delta\bm{p}_t\|_2^2\right).
\end{aligned}
\end{equation}

Here, (b) is due to the constraint on output probabilities: $\sum_{i\in\mathbb{V}}p_{t,i} \equiv 1$. Now assume the estimated advantage is positive: $\hat{A}(x,y) > 0$. The corresponding probability changes for the entire response are then all positive: $\min_t \delta p_{t,y_t} > 0$. To simplify the expression, we define the \emph{effective token probability change} as the weighted average of all independent token probabilities:

\begin{equation}
    \overline{\delta p_{y}} = \frac{\sum_{t=1}^{|y|} \delta p_{t,y_t} \cdot \left(\log p_{t,y_t} + \sum_{i\in\mathbb{V}\backslash\{y_t\}}\frac{\delta p_{t,i}}{\delta p_{t,y_t}}\cdot\log p_{t,i}\right)}{\sum_{t=1}^{|y|} \log p_{t,y_t} + \sum_{i\in\mathbb{V}\backslash\{y_t\}}\frac{\delta p_{t,i}}{\delta p_{t,y_t}}\cdot\log p_{t,i}} > 0.
\end{equation}

The entropy change can then be simplified as:

\begin{equation}
\begin{aligned}
    \Delta\mathcal{H} &\approx -\frac{1}{|y|} \cdot \overline{\delta p_{y}} \cdot \left(\log\prod_{t=1}^{|y|}p_{t,y_t} - \log\prod_{t=1}^{|y|}\prod_{i\in\mathbb{V}\backslash\{y_t\}}p_{t,i}^{-\frac{\delta p_{t,i}}{\delta p_{t,y_t}}}\right) \\
    &= -\frac{1}{|y|} \cdot \overline{\delta p_{y}} \cdot \left(\log\pi_{\bm{\theta}}(y|x) - \log\prod_{t=1}^{|y|}\prod_{i\in\mathbb{V}\backslash\{y_t\}}p_{t,i}^{-\frac{\delta p_{t,i}}{\delta p_{t,y_t}}}\right),
\end{aligned}
\end{equation}

showing that entropy change is negatively related to advantage when the likelihood of the generated response exceeds a threshold: $\pi_{\bm{\theta}}(y|x) > \prod_{t=1}^{|y|}\prod_{i\in\mathbb{V}\backslash\{y_t\}} p_{t,i}^{-\frac{\delta p_{t,i}}{\delta p_{t,y_t}}}$. The same derivation applies to the negative advantage case ($\hat{A(x,y)<0}$). Therefore, $\hat{A}(x,y) \cdot \Delta\mathcal{H} \leq 0$ holds when the rollout sample likelihood is sufficiently high.

\hfill $\square$


\section{Additional Experiment Details}
\subsection{Implementation Details}\label{appendix:implement_detail}
The implementation is based on the verl framework~\citep{sheng2025hybridflow}, which uses vLLM~\citep{kwon2023efficient} for inference and FSDP2~\citep{zhao2023pytorch} for training. All experiments are conducted on a node of 8 NVIDIA H100 GPUs. We summarize the core hyperparameters in Table~\ref{tab:hyperparameter}. The average training time for \texttt{Qwen3-4B-Base} is around 10K training samples per hour.

\begin{table}[h!]
\caption{Core hyperparameters for training implementation. The origin of each hyperparameter is listed, and all empty-source hyperparameters are determined by our experiments.}
\label{tab:hyperparameter}
\centering
\begin{tabular}{lll}
\toprule
\textbf{Hyperparameter}          & \textbf{Value} & \textbf{Source}             \\ \midrule
\texttt{train\_batch\_size}      & 1024           & verl examples               \\
\texttt{ppo\_mini\_batch\_size}  & 8$\times$32    & verl examples               \\
\texttt{max\_context\_length}    & 1024 + 3072    & \citet{xiong2025minimalist} \\
\texttt{rollout.n}               & 8              &                             \\ \midrule
\texttt{optim.lr}                & 1e-6           & verl examples               \\
\texttt{kl\_loss\_coef}          & 1e-3           & verl examples               \\ \midrule
\texttt{val\_kwargs.n}           & 32             &                             \\
\texttt{val\_kwargs.temperature} & 0.6            &                             \\ \bottomrule
\end{tabular}
\end{table}

\subsection{Entropy Curves of Baselines}\label{appendix:entropy_curve}
We show the full entropy curves of baselines in Fig.~\ref{fig:baseline_entropy}. Standard GRPO~\citep{shao2024deepseekmath} exhibits entropy collapse. On top of GRPO, many entropy-preserving methods successfully alleviate the entropy-decreasing trend. However, they are not necessarily responsive enough to enable accurate entropy control and may cause the entropy curves to be unstable in the long term. In contrast, the proposed Entrocraft accurately stabilizes entropy at the target (0.8), making exploration controllable as a hyperparameter. 

\begin{figure}[h!]
    \centering
    \includegraphics[width=\linewidth]{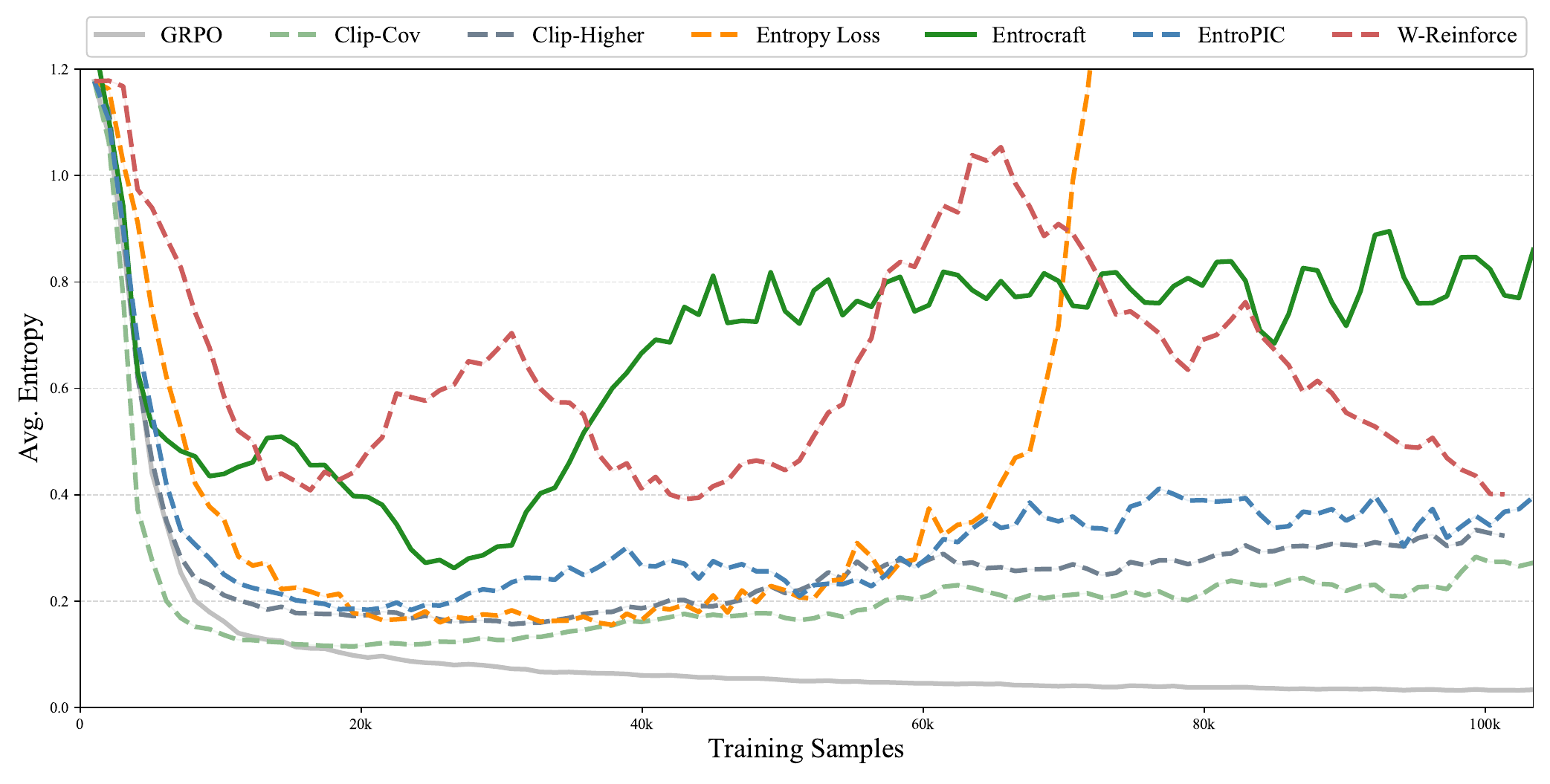}
    \caption{Entropy curve comparison across baselines. Other entropy-preserving methods may induce instability during long-term training or may not be sufficiently responsive. In contrast, the proposed Entrocraft effectively controls the entropy curve to be the customized shape.}
    \label{fig:baseline_entropy}
\end{figure}

\subsection{Effective Batch Size}\label{appednix:batch}
To explicitly show the effect of Entrocraft on the number of rollout samples used for gradient update, we pick the steps that trigger rejection sampling, and show their effective batch sizes throughout the training in Fig~\ref{fig:batch}. The RL training is GRPO + Entrocraft with initial \texttt{rollout.n=8}, and the entropy range setting is the linear decaying as used in Section~\ref{sec:exp_anneal}. The dropping effective batch sizes verify that Entrocraft requires less gradient computation than naive GRPO.

\begin{figure}[h!]
    \centering
    \includegraphics[width=\linewidth]{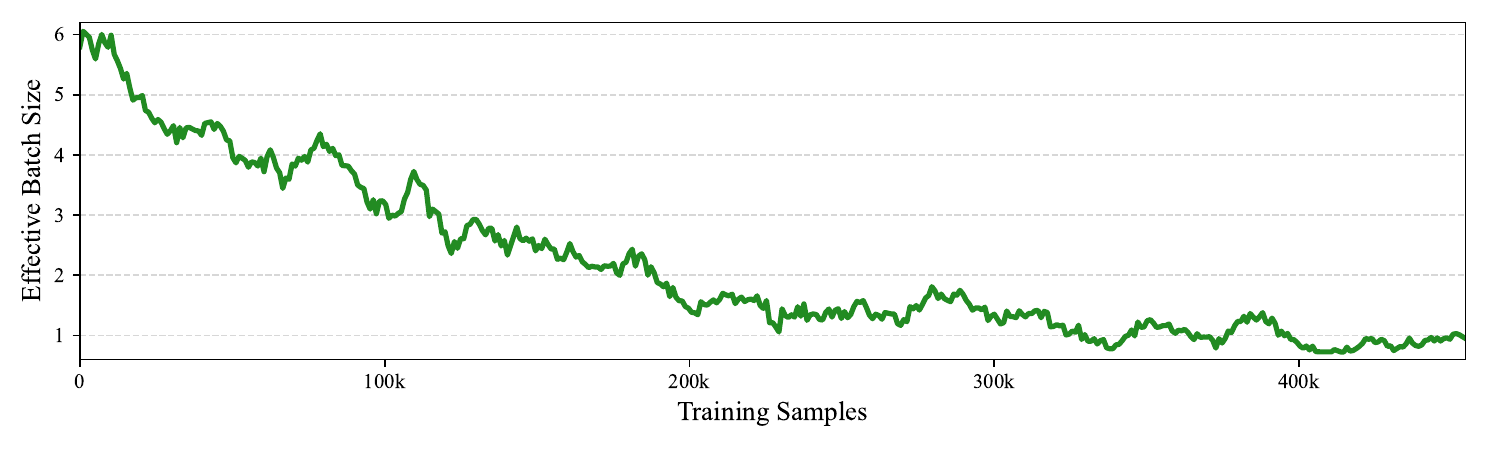}
    \caption{Effective batch sizes for the RL steps affected by Entrocraft. The negative samples become less as the model's performance improves.}
    \label{fig:batch}
\end{figure}

\subsection{Full Evaluation on AIME}\label{appendix:aime}
In addition to the benchmark evaluations of Table~\ref{tab:main}, we list the full AIME results in Table~\ref{tab:full_aime}. The conclusions on all AIME benchmarks are consistent with Table~\ref{tab:main}. Entrocraft outperforms all other entropy-preserving methods across diverse RL algorithms.

\begin{table}[h!]
\caption{Full Evaluations on AIME. The proposed Entrocraft consistently improves the final performance of RL algorithms, better than other entropy-preserving methods. The best scores are in \textbf{bold}, and the second-best scores are marked with \underline{underlines}.}
\label{tab:full_aime}
\centering
\begin{tabular}{l|cc|cc|cc}
\toprule
\multicolumn{1}{c|}{\multirow{2}{*}{\textbf{Method}}} & \multicolumn{2}{c|}{\textbf{AIME-24}} & \multicolumn{2}{c|}{\textbf{AIME-25}} & \multicolumn{2}{c}{\textbf{AIME-26}} \\
\multicolumn{1}{c|}{}                                 & mean@32           & pass@32           & mean@32           & pass@32           & mean@32           & pass@32          \\ \midrule
GRPO                                                  & 7.7               & \underline{30.0}  & 8.9               & 40.0              & 8.1               & 33.3             \\
~~~~ + Entropy Loss                                   & 9.0               & 26.7              & 9.9               & 33.3              & 9.1               & 26.7             \\
~~~~ + Clip-Higher                                    & \underline{10.2}  & 23.3              & 10.8              & 40.0              & 9.1               & \underline{36.7} \\
~~~~ + Clip-Cov                                       & 9.7               & 26.7              & 10.8              & 33.3              & 10.9              & \underline{36.7} \\
\rowcolor{green!15!white}~~~~ + Entrocraft            & \textbf{16.7}     & \textbf{36.7}     & \textbf{16.5}     & \textbf{46.7}     & \textbf{15.2}     & \textbf{40.0}    \\ \midrule
GSPO                                                  & 8.8               & 26.7              & 8.9               & 33.3              & 8.8               & 33.3             \\
~~~~ + Entropy Loss                                   & 7.9               & 23.3              & 11.9              & 33.3              & 9.6               & \underline{36.7} \\
~~~~ + Clip-Higher                                    & 10.0              & 20.0              & 10.0              & 40.0              & 9.5               & \underline{36.7} \\
~~~~ + Clip-Cov                                       & 9.7               & 26.7              & 10.0              & 40.0              & 9.1               & \underline{36.7} \\
\rowcolor{green!15!white}~~~~ + Entrocraft            & \underline{10.2}  & 26.7              & \underline{14.7}  & 40.0              & \underline{13.3}  & \underline{36.7} \\ \midrule
W-Reinforce                                           & 8.8               & 26.7              & 11.3              & 40.0              & 9.3               & 26.7             \\
EntroPIC                                              & 9.9               & 20.0              & 10.9              & \underline{43.3}  & 9.7               & 30.0             \\ \bottomrule
\end{tabular}
\end{table}

\subsection{Results on Other Models}
Aside from \texttt{Qwen3-4B-Base}, we also try Entrocraft on larger models and models from different model families, as shown in Table~\ref{tab:llama8}, \ref{tab:qwen8}, and \ref{tab:qwen14}. The results report consistent improvement over GRPO across different base models.

\begin{table}[h!]\setlength\tabcolsep{1.9pt}
\caption{Benchmark results on \texttt{Llama-3.1-8B-Instruct}. Entrocraft demonstrates consistent improvements over GRPO. The best scores are in \textbf{bold}.}
\label{tab:llama8}
\centering
\begin{tabular}{@{}l|cc|cc|cc|cc@{}}
\toprule
\multicolumn{1}{c|}{\multirow{2}{*}{\textbf{Method}}} & \multicolumn{2}{c|}{\textbf{MATH-500}} & \multicolumn{2}{c|}{\textbf{AMC-23}} & \multicolumn{2}{c|}{\textbf{AIME-25}} & \multicolumn{2}{c}{\textbf{Avg. Score}} \\
\multicolumn{1}{c|}{}                                 & mean@32            & pass@32           & mean@32           & pass@32          & mean@32           & pass@32           & mean@32            & pass@32            \\ \midrule
GRPO                                                  & 56.2               & 82.2              & 30.8              & 75.0             & 10.4              & 16.7              & 32.5               & 58.0               \\
\rowcolor{green!15!white}~~ + Entrocraft              & \textbf{57.1}      & \textbf{85.2}     & \textbf{36.8}     & \textbf{85.0}    & \textbf{18.8}     & \textbf{26.7}     & \textbf{37.6}      & \textbf{65.6}      \\ \bottomrule
\end{tabular}
\end{table}

\begin{table}[h!]\setlength\tabcolsep{1.9pt}
\caption{Benchmark results on \texttt{Qwen3-8B-Base}. Entrocraft demonstrates consistent improvements over GRPO. The best scores are in \textbf{bold}.}
\label{tab:qwen8}
\centering
\begin{tabular}{@{}l|cc|cc|cc|cc@{}}
\toprule
\multicolumn{1}{c|}{\multirow{2}{*}{\textbf{Method}}} & \multicolumn{2}{c|}{\textbf{MATH-500}} & \multicolumn{2}{c|}{\textbf{AMC-23}} & \multicolumn{2}{c|}{\textbf{AIME-25}} & \multicolumn{2}{c}{\textbf{Avg. Score}} \\
\multicolumn{1}{c|}{}                                 & mean@32            & pass@32           & mean@32           & pass@32          & mean@32           & pass@32           & mean@32            & pass@32            \\ \midrule
GRPO                                                  & 77.2               & 91.8              & 59.7              & 92.5             & 13.1              & 43.3              & 50.0               & 75.9               \\
\rowcolor{green!15!white}~~~~ + Entrocraft            & \textbf{78.4}      & \textbf{94.0}     & \textbf{63.5}     & \textbf{97.5}    & \textbf{13.8}     & \textbf{46.7}     & \textbf{51.9}      & \textbf{79.4}      \\ \bottomrule
\end{tabular}
\end{table}

\begin{table}[h!]\setlength\tabcolsep{1.9pt}
\caption{Benchmark results on \texttt{Qwen3-14B-Base}. Entrocraft demonstrates consistent improvements over GRPO. The best scores are in \textbf{bold}.}
\label{tab:qwen14}
\centering
\begin{tabular}{@{}l|cc|cc|cc|cc@{}}
\toprule
\multicolumn{1}{c|}{\multirow{2}{*}{\textbf{Method}}} & \multicolumn{2}{c|}{\textbf{MATH-500}} & \multicolumn{2}{c|}{\textbf{AMC-23}} & \multicolumn{2}{c|}{\textbf{AIME-25}} & \multicolumn{2}{c}{\textbf{Avg. Score}} \\
\multicolumn{1}{c|}{}                                 & mean@32            & pass@32           & mean@32           & pass@32          & mean@32           & pass@32           & mean@32            & pass@32            \\ \midrule
GRPO                                                  & 80.8               & 92.2              & 66.7              & 97.5             & 14.7              & 40.0              & 54.1               & 76.6               \\
\rowcolor{green!15!white}~~~~ + Entrocraft            & \textbf{81.7}      & \textbf{93.8}     & \textbf{67.1}     & \textbf{97.5}    & \textbf{17.0}     & \textbf{53.3}     & \textbf{55.3}      & \textbf{81.5}      \\ \bottomrule
\end{tabular}
\end{table}

\subsection{Failure Case: Maintaining High Entropy May Induce Instability in The Long Term}\label{abl_study:failure_case}
To discuss the boundary of entropy control methods, we use Entrocraft with different entropy targets and record their training dynamics in long-term RL. We show two distinct examples in Fig.~\ref{fig:failure}. We find that an overly high entropy level introduces instability into RL training, and such instability will accumulate and become more catastrophic in the long term. This also validates our design of entropy curve annealing schemes.

\begin{figure}[h!]
    \centering
    \subfloat{\includegraphics[width=.33\columnwidth]{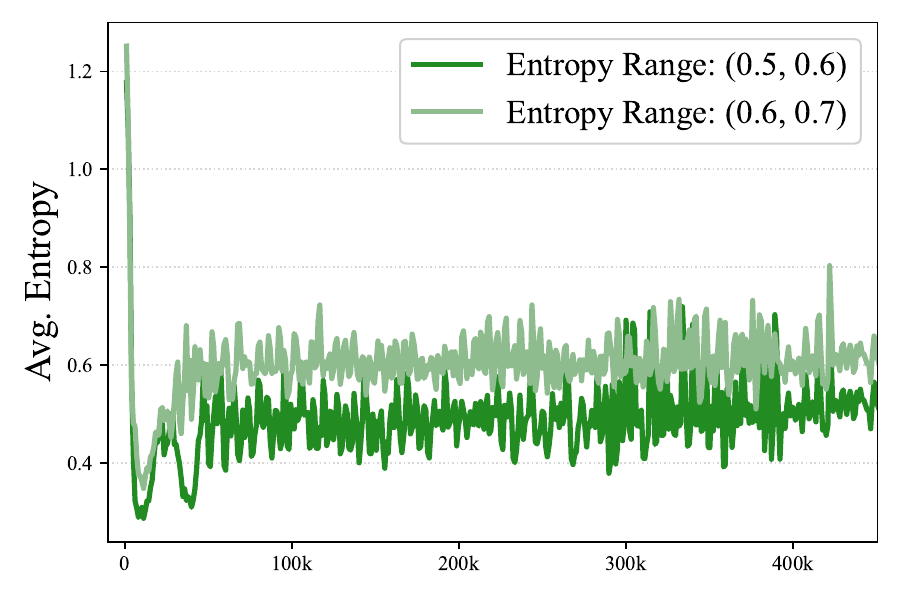}}
    \subfloat{\includegraphics[width=.33\columnwidth]{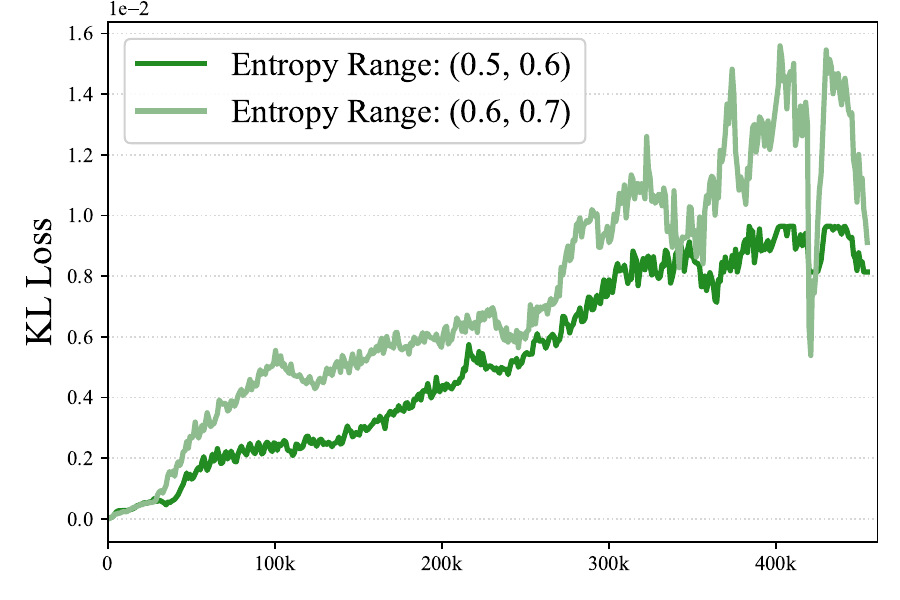}}
    \subfloat{\includegraphics[width=.33\columnwidth]{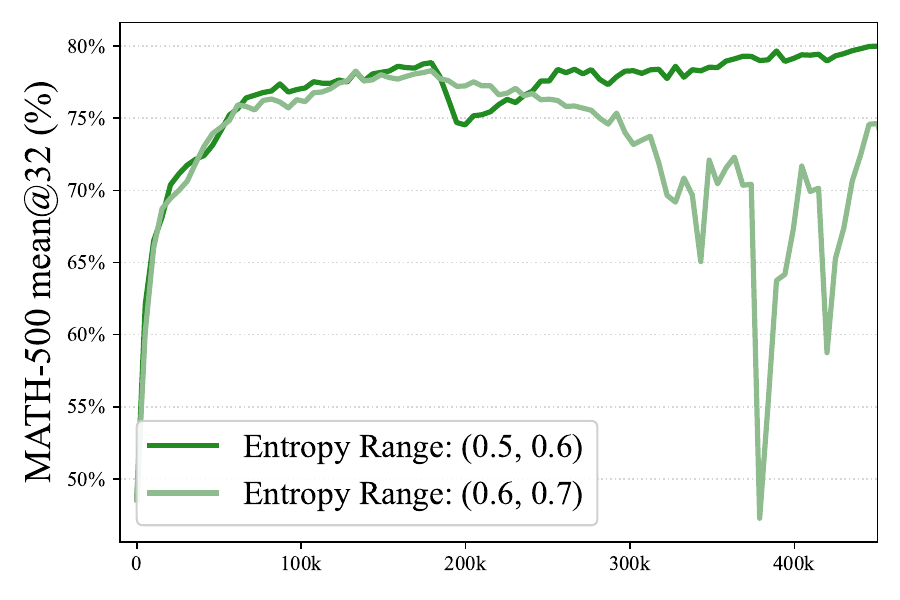}}
    \caption{Failure case study. Overly high entropy level introduces instability into RL training, making the empirical performance fragile to subtle perturbations.}
    \label{fig:failure}
\end{figure}






\end{document}